\pdfoutput=1 
\documentclass{article} 

\makeatletter
\let\mgpo@origaddcontentsline\addcontentsline
\makeatother

\usepackage{iclr2027_conference,times}

\makeatletter
\let\addcontentsline\mgpo@origaddcontentsline
\makeatother

\usepackage{enumitem}

\usepackage[utf8]{inputenc} 
\usepackage[T1]{fontenc}    
\usepackage{hyperref}       
\usepackage{url}            
\usepackage{booktabs}       
\usepackage{amsfonts}       
\usepackage{nicefrac}       
\usepackage{microtype}      
\usepackage{xcolor}         
\usepackage{multirow}
\usepackage{arydshln} 
\usepackage{color , colortbl}
\usepackage{placeins}
\usepackage{algorithm}
\usepackage{algorithmic}
\usepackage{graphicx}
\usepackage{amsmath, amssymb, amsthm, mathtools}
\usepackage{cleveref}
\usepackage{caption}
\usepackage{titletoc}
\hypersetup{colorlinks=true,linkcolor=red!60!black,citecolor=blue!60!black,urlcolor=magenta!60!black}
\newcommand{\mgpo}{\texttt{MGPO}}
\newcommand{\ft}{\textsc{Pixel Reward}}  
\newcommand{\sss}{\textsc{MST Reward}}   
\definecolor{tabhighlight}{rgb}{0.925,1,1}

\theoremstyle{plain}
\newtheorem{theorem}{Theorem}
\newtheorem{lemma}[theorem]{Lemma}
\newtheorem{proposition}[theorem]{Proposition}
\newtheorem{corollary}[theorem]{Corollary}

\theoremstyle{definition}

\theoremstyle{remark}
\newtheorem{remark}[theorem]{Remark}

\newtheorem*{remark*}{Remark}   
\newtheorem*{note*}{Note}
\newcommand{\manifold}{\mathcal{M}}        
\newcommand{\R}{\mathbb{R}}                

\DeclareMathOperator{\reach}{reach}        
\DeclareMathOperator*{\argmin}{arg\,min}

\newcommand{\defeq}{\coloneqq}             

\let\cite\citep

\title{MGPO: Manifold-Guided Diffusion Alignment for Task-Aware Dataset Distillation}

\author{
  Yunyi Chen$^{1,2}$\thanks{Equal contribution.}\quad
  Chenru Wang$^{1}$\footnotemark[1]\quad
  Xinyi Ye$^{1}$\quad
  Zexin Zheng$^{3}$\quad
  Chi Zhang$^{1}$\thanks{Corresponding author.} \\[2pt]
  {\normalfont $^{1}$AGI Lab, Westlake University\quad
  $^{2}$Eindhoven University of Technology} \\
  {\normalfont $^{3}$Guangdong University of Finance and Economics}
}

\iclrfinalcopy 
\begin{document}

\maketitle
\lhead{}\renewcommand{\headrulewidth}{0pt} 
\vspace{-4.5mm}
\begin{figure}[h]
    \centering
    \includegraphics[width=1.0\linewidth]{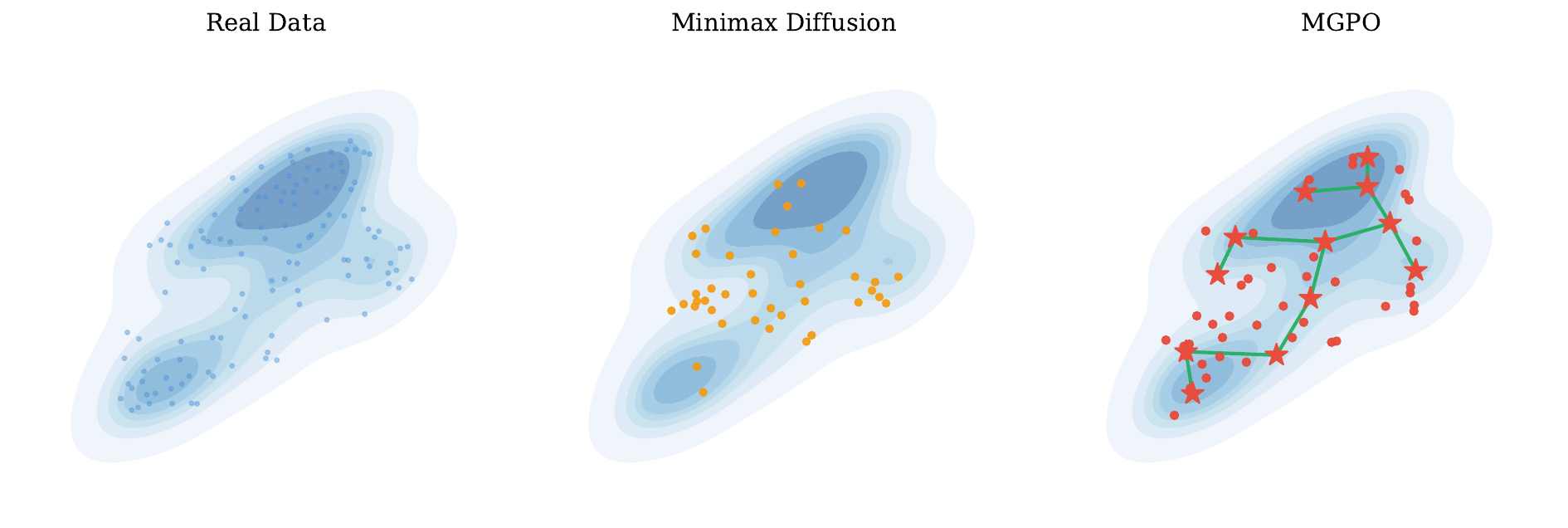}
\caption{
t-SNE of latent distributions on ImageWoof. Blue density: real data; orange dots: 
Minimax Diffusion; red dots: \mgpo~(ours); red stars with green edges: class-wise 
MST skeleton. \mgpo~covers a broader area along the MST skeleton, indicating 
improved manifold coverage.
}
\label{fig:tsne-mst}
\end{figure}

\begin{abstract}
Diffusion-based dataset distillation (DD) suffers from a fundamental 
objective mismatch: likelihood-driven diffusion models prioritize density 
approximation over the discriminative decision boundaries required for 
downstream tasks. Beyond semantic mismatch, relying solely on 
density also leads to geometric coverage loss, where generated samples collapse 
into a few high-density modes and fail to cover the manifold's structural 
diversity. We propose Manifold-Guided Policy Optimization (\mgpo), 
which reformulates DD as a multi-objective reinforcement learning problem 
and achieves Dual-Space Alignment via a pixel-space discriminative 
reward and a latent-space geometric reward guided by a class-wise Minimum 
Spanning Tree (MST). The discriminative reward enforces class 
separability, while the MST-based geometric reward encourages generated latents to cover a sparse geometric skeleton of each class, jointly addressing both failure 
modes. We further provide an idealized analysis that motivates the MST-based reward, including a Hausdorff approximation bound and a subsampling bound 
independent of the dataset size. The reward-modular design extends to structured tasks such as object detection and segmentation  by substituting the frozen task reward model. 
Extensive experiments show \mgpo~consistently outperforms existing methods, 
including a +8.0\% mIoU gain on segmentation under low-budget settings.
\end{abstract}
\section{Introduction}
\label{sec:intro}

Dataset distillation (DD)~\cite{wang2018dataset} has emerged as a critical technique in large-scale machine learning, aiming to synthesize a small, representative subset of data from a massive original dataset. Recently, generative models, particularly diffusion models~\cite{ddpm,DDIM}, have emerged as effective backbones for this task. By leveraging powerful generative priors, diffusion models can effectively capture the complex underlying distribution of the source data to synthesize diverse and high-fidelity images. Unlike traditional gradient-matching~\cite{GM} or trajectory-matching~\cite{MTT} approaches, which often struggle with optimization stability and limited sample diversity, these generative frameworks naturally overcome such bottlenecks. Consequently, diffusion-based methods have established state-of-the-art (SOTA) performance in dataset distillation~\cite{minimax,D4M,cao2,mgd3,ddvlcp,unidd}, offering superior results within limited computational budgets.

Despite recent advances, diffusion-based dataset distillation under extremely low-budget settings faces two distinct challenges. First, a \textit{semantic objective mismatch}: likelihood-driven diffusion models prioritize pixel-level density approximation over the discriminative features required by downstream tasks~\cite{biroli2024dynamical,ddpm,karczewski2025diffusionmodelscartoonistscurious}. Second, a \textit{geometric coverage loss}: relying solely on density or naive classifier guidance often leads to mode collapse, failing to capture the complex geometric manifold structure and structural diversity of each class in the original data distribution. Consequently, generated samples may exhibit high perceptual fidelity yet fail to effectively represent the dataset.

These challenges show that diffusion-based distillation must transcend perceptual fidelity by resolving both issues at the same time. An effective distilled dataset requires two orthogonal properties: semantic discriminability, which pushes samples toward highly separable regions to better serve downstream tasks; and geometric preservation, which spreads samples along a sparse geometric skeleton of each class to maintain structural diversity and prevent mode collapse~\cite{guo2024gradient}.

Following these principles, we propose Manifold-Guided Policy Optimization (\mgpo), a framework that reformulates diffusion-based dataset distillation as a multi-objective reinforcement learning problem. Our core contribution is a Dual-Space Alignment strategy that addresses both challenges in their respective domains. We design a pixel-space discriminative reward to enforce class separability, alongside a latent-space geometric reward guided by a class-wise Minimum Spanning Tree (MST) to encourage geometric coverage. Crucially, \mgpo~utilizes a decoupled group-wise policy optimization strategy to internalize these heterogeneous, non-differentiable rewards directly into the diffusion weights. This RL formulation eliminates the need for expensive inference-time guidance, ensuring highly efficient sampling while harmonizing semantics and geometry.

We further provide an idealized analysis showing that, under smoothness and reach assumptions, a Euclidean MST $\delta$-approximates the class manifold in Hausdorff distance with $\delta^2/\tau$ tightness in the backward direction. This analysis motivates the distance term of the geometric reward; the implemented reward is evaluated empirically.

Beyond classification, the flexibility of our RL formulation unlocks a new reward-modular paradigm for dataset distillation. While the MST-based geometric prior is inherently task-agnostic, the outer framework can be adapted to complex downstream tasks, such as object detection and segmentation, by substituting the pixel-space reward model. In such scenarios, the same frozen task model provides the reward during fine-tuning and pseudo-annotates the class-conditioned generated images afterwards, without modifying the core generative architecture.

Extensive empirical evaluations across multiple benchmarks demonstrate that \mgpo~consistently outperforms existing SOTA approaches. For instance, under representative low-budget settings, our method achieves 18.3 mIoU on segmentation tasks, surpassing the strongest baseline by +8.0 mIoU.

Our contributions are summarized as follows:

\begin{itemize}[leftmargin=4mm]
    \item We reformulate diffusion-based dataset distillation as a multi-objective reinforcement learning problem, internalizing task-specific rewards directly into the diffusion weights and eliminating inference-time guidance overhead.
    
    \item We design a dual-space reward with a pixel-level discriminative term and a latent-level MST geometric term, addressing the semantic mismatch and geometric coverage loss.
    
    \item We provide an idealized analysis motivating the MST-based reward: under smoothness and reach assumptions, (i) a Euclidean MST approximates the class manifold in Hausdorff distance, and (ii) a subsampled MST admits a bound independent of the dataset size. 
    
    \item Extensive experiments show \mgpo~achieves SOTA performance on standard classification benchmarks and extends to other complex structured tasks like object detection and semantic segmentation.
\end{itemize}

\begin{figure*}[t]
    \centering
    \includegraphics[width=\linewidth]{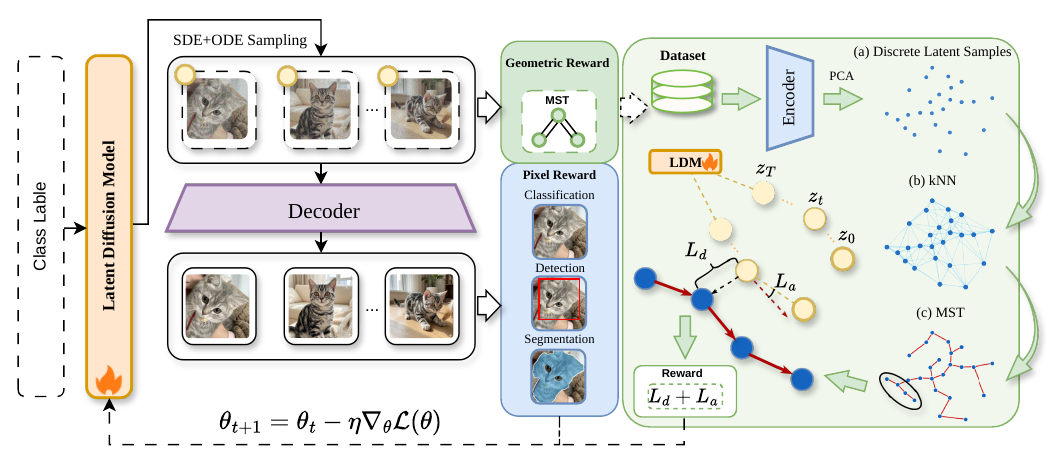}
    \caption{Overview of the proposed~\mgpo~~framework. 
(a)~\ft~($r_{\text{disc}}$) is computed from downstream objectives such as classification, detection, and segmentation, encouraging semantic correctness and task consistency.
(b)~\sss~($r_{\text{geo}}$) is constructed in latent space. Real dataset features are encoded by 
encoder and projected via PCA to obtain a structured latent representation. A kNN graph is built to approximate local neighborhood structure, followed by a Minimum Spanning Tree (MST) to capture the global manifold topology. The generated latent samples are aligned with this manifold structure, promoting both representativeness and structural consistency.
}
    \label{fig:pipeline}

\end{figure*}
\vspace{-20pt}
\section{Related Work}

\subsection{Diffusion Models for Dataset Distillation}

Recent work has explored diffusion models as powerful generative backbones for DD. 
Minimax~\cite{minimax} fine-tunes pre-trained diffusion models under a minimax objective to balance sample representativeness and diversity. 
D$^4$M~\cite{D4M} incorporates text-to-image diffusion with prototype learning to improve intra-class consistency. 
D3HR~\cite{D3HR} addresses distribution mismatch via DDIM inversion~\cite{DDIM} and group-wise sampling, aligning generated and real data distributions. 
CaO2~\cite{cao2} performs inference-time latent optimization to reduce feature level inconsistencies, while DDVLCP~\cite{ddvlcp} leverages vision–language prototypes to enhance semantic coherence. Despite their success, diffusion-based approaches are inherently likelihood-centric, concentrating samples in high-density regions, which preserves perceptual fidelity but neglects boundary and hard examples essential for discriminative training. Most related to our work, ManifoldGD~\cite{roy2026manifoldgd} explores manifold structure for diffusion-based distillation, but relies on inference-time guidance without formal geometric guarantees. In contrast, \textbf{~\mgpo~} internalizes geometric constraints into the diffusion policy via RL,  with an idealized Hausdorff analysis motivating the MST-based reward, while simultaneously resolving the semantic objective mismatch through dual-space alignment.

\subsection{Reinforcement Learning in Diffusion Models}

Reinforcement learning has recently been introduced to align diffusion and flow-based generative models with task-level objectives beyond likelihood-based training. 
A representative approach is Group Relative Policy Optimization (GRPO)~\cite{shao2024deepseekmathpushinglimitsmathematical}, which stabilizes policy updates by comparing relative rewards within groups of sampled trajectories. Subsequent works extend GRPO to generative settings. 
Flow-GRPO~\cite{liu2025flowgrpo} adapts it to diffusion and flow matching models, while Temp-Flow-GRPO~\cite{tempflow} introduces time-dependent credit assignment. 
Mix-GRPO~\cite{li2026mixgrpounlockingflowbasedgrpo} combines multiple sampling dynamics, and structured variants such as Branch-GRPO~\cite{li2025branchgrpostableefficientgrpo} and Dance-GRPO~\cite{xue2025dancegrpo} improve scalability through branching and group-wise rollouts. However, existing methods focus primarily on generative quality or preference alignment. 
In contrast, we leverage reinforcement learning to directly optimize downstream discriminative objectives. More related works see Appendix~\ref{sec: related}
\section{Preliminary}

\subsection{Dataset Distillation}
Dataset distillation aims to synthesize a small and compact dataset
\(
\mathcal{S} = \{(\tilde{x}_i, \tilde{y}_i)\}_{i=1}^{N_S}
\)
that preserves the essential training information of a large-scale real dataset
\(
\mathcal{T} = \{(x_i, y_i)\}_{i=1}^{N_T},
\)
where \(x_i\) denotes an input sample (e.g., an image), \(y_i\) its corresponding label, and \(N_S \ll N_T\).
The goal of dataset distillation is to enable models trained on \(\mathcal{S}\) to achieve performance comparable to those trained on \(\mathcal{T}\), while significantly reducing storage and computational costs.

\subsection{DDIM Sampling in Latent Space}
\label{subsec: DDIM}
We adopt a latent diffusion formulation,
where $z_0 = E(x_0)$ denotes the VAE-encoded latent,
and diffusion is performed directly on the latent variables $z_t$.
The DiT backbone parameterizes the noise prediction
$\epsilon_\theta(z_t, t)$ in this latent space.
We model the reverse diffusion dynamics using the generalized sampling process of Denoising Diffusion Implicit Models~(DDIM)~\cite{DDIM}. The transition from state $z_t$ to $z_{t-1}$ is defined as:
\begin{equation}
    z_{t-1} = \sqrt{\alpha_{t-1}} \left( \frac{z_t - \sqrt{1-\alpha_t}\epsilon_\theta(z_t, t)}{\sqrt{\alpha_t}} \right) + \sqrt{1 - \alpha_{t-1} - \sigma_t^2} \cdot \epsilon_\theta(z_t, t) + \sigma_t \epsilon_t,
\end{equation}
where $\epsilon_t \sim \mathcal{N}(0, I)$ is standard Gaussian noise and $\alpha_t$ follows the standard variance schedule. This non Markovian formulation unifies deterministic and stochastic sampling through the variance parameter $\sigma_t= \eta \sqrt{\frac{1 - \alpha_{t-1}}{1 - \alpha_t}} \sqrt{1 - \frac{\alpha_t}{\alpha_{t-1}}}$, which is parameterized by a scalar $\eta$.
The hyperparameter $\eta$ explicitly controls the degree of stochasticity. When $\eta = 0$, the random noise term vanishes, reducing the sampling process to the strictly deterministic DDIM formulation. Conversely, when $\eta = 1$, the process recovers the stochastic Markovian transitions of traditional DDPM~\cite{ddpm}.

During the reward driven policy optimization phase, we maintain $\eta = 1$ to encourage stochastic trajectory exploration. Once the generative policy is fully aligned with our dual space objectives, we set $\eta = 0$ to perform highly efficient and deterministic dataset synthesis.

\section{Method}
\label{sec:method}

To bridge the gap between likelihood driven diffusion generation and discriminative downstream tasks, \textbf{~\mgpo~} internalizes dual space task objectives into the diffusion backbone via RL.

\subsection{Dual-Space Reward Design}
\label{subsec:rewards}

To align diffusion generation with downstream objectives, we design 
two complementary rewards in different spaces. The discriminative 
reward encourages semantic separability in pixel space, while the geometric 
reward encourages coverage of class-level geometric structure in latent space.

\noindent\textbf{Discriminative Reward ($r_{disc}$).}
To ensure the synthesized samples are informative for training downstream
models, we employ a pre-trained proxy classifier $f_\phi$.
For a generated sample $\tilde{x}_0$ with target label $c$, the reward is defined
as the log probability of the ground-truth class:

\begin{equation}
r_{disc}(\tilde{x}_0, c)
=
\log
\left(
\frac{\exp(f_\phi(\tilde{x}_0)_c)}
{\sum_k \exp(f_\phi(\tilde{x}_0)_k)}
\right),
\label{r_disc}
\end{equation}

where $f_\phi(\tilde{x}_0)_c$ denotes the predicted logit for class $c$.
This continuous formulation provides a smooth reward landscape and penalizes
samples that lie near ambiguous decision boundaries.

\noindent\textbf{Manifold Geometric Reward ($r_{\mathrm{geo}}$).}
Likelihood-driven diffusion models tend to concentrate samples in high-density 
regions, which may overlook geometrically informative 
regions of the class manifold. To encourage geometric coverage of each class, we introduce 
a geometric reward grounded in a class-wise Minimum Spanning Tree (MST), which provides a sparse geometric skeleton of 
the class distribution and is more discriminative than density-based regularization.

Concretely, let $\mathcal{Z}_c = \{z_i\}_{i=1}^{N_c}$ denote the VAE-encoded 
latent representations of class-$c$ samples. We model the class manifold via a 
weighted graph $G_c = (\mathcal{Z}_c, E_c)$ with edge weights 
$w(u, v) = \|u - v\|_2$, and define the class-wise MST as
\begin{equation}
    \mathcal{T}_c = \arg\min_{T \subseteq G_c} \sum_{(u, v) \in T} w(u, v),
\end{equation}
where $T$ spans all nodes and contains no cycles. To ensure scalability, in 
practice we construct $\mathcal{T}_c$ from a subsampled latent set on a sparse 
$k$-nearest-neighbor graph (see Algorithm~\ref{alg:mst}).

Given a generated latent $z_0$ corresponding to $\tilde x_0$, we project it 
into the normalized latent space via $u_{\mathrm{gen}} = g_c(z_0)$, and define 
the geometric reward as

\begin{equation}
\begin{gathered}
r_{\mathrm{geo}}(\tilde x_0, c)
= -\lambda_{\mathrm{dist}}\min_{v \in \mathcal{V}_c}
\|u_{\mathrm{gen}} - v\|_2
-\lambda_{\mathrm{dir}}\bigl(1 -
|\cos(\mathbf{v}_{\mathrm{gen}}, \mathbf{d}_{e^*})|\bigr), \\
e^* = \argmin_{(i,j) \in \mathcal{E}_c}
\bigl\|u_{\mathrm{mid}} - \tfrac{v_i + v_j}{2}\bigr\|_2.
\end{gathered}
\end{equation}

where $\mathcal{V}_c$ and $\mathcal{E}_c$ are the node and edge sets of $\mathcal{T}_c$, the first term is the Euclidean distance from the projected terminal latent to its nearest MST node, and $\lambda_{\mathrm{dir}}$ controls the strength 
of the directional alignment term (analysis in Appendix). The transition 
direction is $\mathbf{v}_{\mathrm{gen}} = g_c(z_0) - g_c(z_{\mathrm{mid}})$, 
where $z_{\mathrm{mid}}$ is a latent sampled from an intermediate diffusion 
timestep, and $\mathbf{d}_{e^*}$ is the direction of the MST edge $e^*$ whose midpoint is nearest to the projected intermediate latent $u_{\mathrm{mid}} = g_c(z_{\mathrm{mid}})$. The first term encourages generated samples to remain close to the manifold 
skeleton, while the second encourages alignment with its local direction. 
By incorporating $r_{\mathrm{geo}}$ into policy optimization, the diffusion 
model internalizes the geometric prior  during fine-tuning, removing the need for inference-time guidance.

\paragraph{Theoretical Justification.}
We provide two idealized results that motivate the MST-based geometric reward and its subsampled construction. Both concern a Euclidean MST built on samples from a smooth class manifold and an edge-distance formulation. They do not cover the VAE/PCA representation, the nearest-node distance used in practice, or the direction term, which we evaluate empirically.
{Let $\mathcal{M}_c$ be a compact, connected $C^2$ class submanifold with reach $\tau_c > 0$, and let the MST nodes cover $\mathcal{M}_c$ with radius $\delta < \tau_c$. Theorem~\ref{thm:mst} (Appendix~\ref{proof:mstasms}) shows that $d_H(\mathcal{M}_c, |\mathcal{T}_c|) \le \delta$, with the sharper backward bound $\sup_{y \in |\mathcal{T}_c|} d(y, \mathcal{M}_c) \le \delta^2/\tau_c$; that is, in this idealized setting the MST is a reach-aware geometric approximation of the class manifold. Theorem~\ref{thm:subsample-mst} (Appendix~\ref{proof:subsample-mst}) shows that an MST built on $n$ i.i.d.\ samples satisfies $d_H(|\mathcal{T}(V_c)|, |\mathcal{T}(U_c)|) = \mathcal{O}\bigl((\log(n/\eta)/n)^{1/m_c}\bigr)$ with probability at least $1-\eta$, where $m_c$ is the intrinsic dimension; this bound is independent of the full dataset size $N_c$, so the per-class MST construction cost stays constant as datasets scale. These results together motivate our design of the implemented reward.}


\subsection{~\mgpo~: Manifold-Guided Policy Optimization}

We formulate diffusion-based dataset distillation as a sequential decision 
process~(see Appendix~\ref{sec:makchain}), where the diffusion model acts as a stochastic policy that gradually 
transforms Gaussian noise into a synthetic sample. At each timestep, the policy 
$\pi_\theta(z_{t-1}\mid z_t,c)$ predicts the reverse diffusion transition 
conditioned on the class label $c$. 

Following the generalized DDIM formulation introduced in Sec.~\ref{subsec: DDIM}, we set 
$\eta=1$ during policy optimization to introduce stochasticity into the sampling process. This injected stochasticity is vital as it enables the policy to explore a broader range within latent space,  while task-driven rewards guide the trajectories toward regions 
that are both discriminative and geometrically structured.

\noindent\textbf{Policy Optimization.}
To optimize the generative policy without training an explicit value network, 
we adopt a relative policy optimization strategy. For each class 
condition $c$, we sample a group of $G$ trajectories $\{z_{0}^{(i)}\}_{i=1}^G$ where $G$ denotes the rollout group size 
(see Sec.~\ref{sec:abl}), and evaluate them using the 
dual-space rewards defined in Sec.~\ref{subsec:rewards}. The diffusion policy is updated by 
maximizing a clipped surrogate objective:
\begin{equation}
\mathcal{J}(\theta) =
\mathbb{E}_{c \sim p(c), \epsilon}
\left[
\frac{1}{G}
\sum_{i=1}^{G}
\min\big(
\rho_i \hat{A}_i,
\mathrm{clip}(\rho_i,1-\epsilon,1+\epsilon)\hat{A}_i
\big)
- \beta\, \mathbb{D}_{\mathrm{KL}}\bigl(\pi_\theta \,\|\, \pi_{\theta_{\mathrm{ref}}}\bigr)
\right],
\label{eq: mgpo_objective}
\end{equation}

where $\rho_i =
\frac{\pi_\theta(z^{(i)}_{t-1}\mid z^{(i)}_t,c)}
{\pi_{\theta_{\text{old}}}(z^{(i)}_{t-1}\mid z^{(i)}_t,c)}$
denotes the importance sampling ratio and $\hat{A}_i$ is the aggregated advantage of the $i$-th trajectory. 

The MST skeleton and the projection function $g_c(\cdot)$ are pre-computed following Algorithm~\ref{alg:mst}. ~\mgpo~ training procedure is summarized in Algorithm~\ref{alg:mgpo}.

\begin{algorithm}[t]
\caption{~\mgpo~ Training Procedure}
\label{alg:mgpo}
\begin{algorithmic}[1]

\REQUIRE Diffusion policy $\pi_\theta$, proxy reward model $f_\phi$, 
class-wise MST skeleton $\mathcal{T}_c$ and projection function $g_c(\cdot)$

\FOR{each training iteration}

\STATE Sample class label $c \sim p(c)$

\STATE Generate $G$ trajectories $\{\tau_i\}_{i=1}^{G}$ via SDE and ODE sampling

\FOR{each trajectory $\tau_i = \{z_T^{(i)},\dots,z_0^{(i)}\}$}

\STATE Compute geometric reward $r_{\text{geo}}^{(i)}$
using latent projection $u_{\text{gen}}^{(i)} = g_c(z_0^{(i)})$

\STATE Decode image $\tilde{x}_0^{(i)} = D(z_0^{(i)})$

\STATE Compute discriminative reward $r_{\text{disc}}^{(i)}$ using (Eq.~\ref{r_disc})

\ENDFOR

\STATE Perform group-wise normalization of rewards (Eq.~\ref{eq: group_normal})

\STATE Aggregate advantages across objectives (Eq.~\ref{eq: aggregation})

\STATE Apply batch-level normalization (Eq.~\ref{eq: batch_normal})

\STATE Update diffusion policy using the ~\mgpo~ objective (Eq.~\ref{eq: mgpo_objective})

\ENDFOR

\end{algorithmic}
\end{algorithm}

\noindent\textbf{Multi-Reward Aggregation.}
A key challenge in dual-space optimization is that semantic and geometric 
rewards may lie on different numerical scales. Directly combining them can lead 
to unstable training. Following GDPO~\cite{liu2026gdpogrouprewarddecouplednormalization}, we adopt a three-stage advantage normalization strategy.

First, for each reward type $k \in \{\text{disc},\text{geo}\}$, we perform 
group-wise normalization across the $G$ rollouts:
\begin{equation}
A_{i,k} =
\frac{
r_k^{(i)} - \mathrm{mean}(\{r_k^{(j)}\}_{j=1}^{G})
}{
\mathrm{std}(\{r_k^{(j)}\}_{j=1}^{G}) + \epsilon
}.
\label{eq: group_normal}
\end{equation}

This step removes scale differences while preserving the relative ranking of 
samples within the group.

Second, the normalized advantages are combined through a weighted summation:
\begin{equation}
A_i^{\text{sum}} =
w_{\text{disc}} A_{i,\text{disc}}
+
w_{\text{geo}} A_{i,\text{geo}}.
\label{eq: aggregation}
\end{equation}
where $w_{\text{disc}}$ and $w_{\text{geo}}$ weight the discriminative and geometric rewards 
(Sec.~\ref{sec:abl}).

Finally, we apply batch-level normalization across all samples in the training batch:
\begin{equation}
\hat{A}_i =
\frac{
A_i^{\text{sum}}
-
\mathrm{mean}(\{A_j^{\text{sum}}\})
}{
\mathrm{std}(\{A_j^{\text{sum}}\}) + \epsilon
}.
\label{eq: batch_normal}
\end{equation}

This normalization stabilizes policy gradients and prevents numerical explosion 
when fine-tuning large diffusion backbones.

\subsection{Reward-Modular Extension to Structured Tasks}
\label{subsec:generalization}

While the geometric reward $r_{\mathrm{geo}}$ in Section~\ref{subsec:rewards} 
is task-agnostic, structured tasks such as object detection and segmentation 
require pixel-space rewards that capture spatial, in addition to semantic, 
correctness. The modular design of \mgpo~allows substituting $r_{\mathrm{disc}}$ 
with such task-specific rewards while keeping the policy optimization, 
geometric reward, and overall training pipeline unchanged. We illustrate this with two examples; each task still requires a 
\emph{task-appropriate} reward design, but the burden is reduced from 
modifying the framework to specifying a scalar reward signal.

\noindent \textbf{Object Detection.} To encourage clear spatial boundaries 
without composite losses, we derive the reward from a pretrained detector. 
For an image $\tilde{x}_0$ conditioned on class $c$, we apply Non-Maximum 
Suppression to filter redundant predictions, and define the reward as the 
maximum remaining confidence score for class $c$:
\begin{equation}
    r_{\mathrm{det}}(\tilde{x}_0, c) = \max\bigl(\mathcal{S}_c(\tilde{x}_0)\bigr),
\label{eq: r_det}
\end{equation}
where $\mathcal{S}_c(\tilde{x}_0)$ denotes the set of post-NMS confidence scores.

\noindent \textbf{Image Segmentation.} Instead of relying on predefined mask 
annotations, we extract the spatial probability map from a pre-trained 
segmentation network. To prevent reward hacking via isolated high-confidence 
pixels, we average the top-$K$ pixel confidences for class $c$:
\begin{equation}
    r_{\mathrm{seg}}(\tilde{x}_0, c) = \frac{1}{K}\sum_{k=1}^{K}\mathrm{TopK}\bigl(\mathcal{P}_c(\tilde{x}_0)\bigr)_k,
\label{eq: r_seg}
\end{equation}
where $\mathcal{P}_c(\tilde{x}_0)$ is the flattened probability map. The 
hyperparameter $K$ enforces a minimum expected pixel area (analysis in Appendix).

In both cases, the reward depends only on the prompted class $c$. Spatial annotations are obtained by class-conditioned image generation followed by pseudo-annotation with the same frozen task model. For detection, we keep the post-NMS predictions of class $c$ (score threshold 0.05, IoU threshold 0.5) as boxes. For segmentation, we keep the pixels whose argmax prediction is $c$ as the mask; the Top-$K$ operation in Eq.~\eqref{eq: r_seg} only computes the scalar reward. Images without a valid class-$c$ prediction are kept with empty annotations (detection) or an all-background mask (segmentation). The current formulation therefore does not model multi-class co-occurrence or inter-class relations (see Appendix~\ref{sec:limitations_impacts}). This reward-modular design 
allows MGPO to extend to new structured tasks through reward specification 
rather than architectural modification, though designing an effective reward 
for each task remains a non-trivial step (see Appendix~\ref{sec: supply} for further design 
considerations).

\section{Experiments}
\label{sec:exps}

\vspace{-5pt}
\begin{table*}[h]
\centering
\caption{Comparison of SOTA dataset distillation methods on ImageWoof under various IPC and backbone settings. The best results are marked as bold, and the second are underlined.}
\resizebox{\linewidth}{!}{
\begin{tabular}{cc|cccccccccc|c|c}
\toprule 
IPC (Ratio) & Model & Random & K-Center & Herding & DiT & Minimax & DM & IDC-1 & D$^4$M & DDVLCP & ManifoldGD & \mgpo~(Ours) & Full \\
\midrule 
\multirow{3}{*}{10~(0.8\%)} 
& ConvNet-6   & 24.3{\tiny$\pm$1.1} & 19.4{\tiny$\pm$0.9} & 26.7{\tiny$\pm$0.5} & 34.2{\tiny$\pm$1.1} & \underline{37.0{\tiny$\pm$1.0}} & 26.9{\tiny$\pm$1.2} & 33.3{\tiny$\pm$1.1} & 29.4{\tiny$\pm$0.9} & 34.8{\tiny$\pm$2.4} & 36.9{\tiny$\pm$0.6}  & \textbf{37.2{\tiny$\pm$0.8}} & 86.4{\tiny$\pm$0.2} \\
& ResNetAP-10 & 29.4{\tiny$\pm$0.8} & 22.1{\tiny$\pm$0.1} & 32.0{\tiny$\pm$0.3} & 34.7{\tiny$\pm$0.5} & 39.2{\tiny$\pm$1.3} & 30.3{\tiny$\pm$1.2} & 39.1{\tiny$\pm$0.5} & 33.2{\tiny$\pm$2.1} & \underline{39.5{\tiny$\pm$1.5}} & 38.3{\tiny$\pm$0.4}  & \textbf{40.9{\tiny$\pm$1.1}} & 87.5{\tiny$\pm$0.5} \\
& ResNet-18   & 27.7{\tiny$\pm$0.9} & 21.1{\tiny$\pm$0.4} & 30.2{\tiny$\pm$1.2} & 34.7{\tiny$\pm$0.4} & 37.6{\tiny$\pm$0.9} & 33.4{\tiny$\pm$0.7} & 37.3{\tiny$\pm$0.2} & 32.3{\tiny$\pm$1.2} & \underline{39.9{\tiny$\pm$2.6}} & 39.2{\tiny$\pm$0.7}  & \textbf{40.1{\tiny$\pm$0.3}} & 89.3{\tiny$\pm$1.2} \\
\midrule
\multirow{3}{*}{20~(1.6\%)} 
& ConvNet-6   & 29.1{\tiny$\pm$0.7} & 21.5{\tiny$\pm$0.8} & 29.5{\tiny$\pm$0.3} & 36.1{\tiny$\pm$0.8} & 37.6{\tiny$\pm$0.2} & 29.9{\tiny$\pm$1.0} & 35.5{\tiny$\pm$0.8} & 34.0{\tiny$\pm$2.3} & \underline{37.9{\tiny$\pm$1.9}} & 37.7{\tiny$\pm$0.6} & \textbf{39.6{\tiny$\pm$1.3}} & 86.4{\tiny$\pm$0.2} \\
& ResNetAP-10 & 32.7{\tiny$\pm$0.4} & 25.1{\tiny$\pm$0.7} & 34.9{\tiny$\pm$0.1} & 41.1{\tiny$\pm$0.8} & \underline{45.8{\tiny$\pm$0.5}} & 35.2{\tiny$\pm$0.6} & 43.4{\tiny$\pm$0.3} & 40.1{\tiny$\pm$1.6} & 44.5{\tiny$\pm$2.2}  & 45.6{\tiny$\pm$0.5} & \textbf{46.1{\tiny$\pm$0.3}} & 87.5{\tiny$\pm$0.5} \\
& ResNet-18   & 29.7{\tiny$\pm$0.5} & 23.6{\tiny$\pm$0.3} & 32.2{\tiny$\pm$0.6} & 40.5{\tiny$\pm$0.5} & 42.5{\tiny$\pm$0.6} & 29.8{\tiny$\pm$1.7} & 38.6{\tiny$\pm$0.2} & 38.4{\tiny$\pm$1.1} & \underline{44.5{\tiny$\pm$2.0}}  & 42.4{\tiny$\pm$0.4} & \textbf{45.9{\tiny$\pm$0.2}} & 89.3{\tiny$\pm$1.2} \\
\midrule
\multirow{3}{*}{50~(3.8\%)} 
& ConvNet-6   & 41.3{\tiny$\pm$0.6} & 36.5{\tiny$\pm$1.0} & 40.3{\tiny$\pm$0.7} & 46.5{\tiny$\pm$0.8} & 53.9{\tiny$\pm$0.6} & 44.4{\tiny$\pm$1.0} & 43.9{\tiny$\pm$1.2} & 47.4{\tiny$\pm$0.9} & \underline{54.5{\tiny$\pm$0.6}} & 51.3{\tiny$\pm$0.2} & \textbf{55.7{\tiny$\pm$0.4}} & 86.4{\tiny$\pm$0.2} \\
& ResNetAP-10 & 47.2{\tiny$\pm$1.3} & 40.6{\tiny$\pm$0.4} & 49.1{\tiny$\pm$0.7} & 49.3{\tiny$\pm$0.2} & 56.3{\tiny$\pm$1.0} & 47.1{\tiny$\pm$1.1} & 48.3{\tiny$\pm$1.0} & 51.7{\tiny$\pm$3.2} & 57.3{\tiny$\pm$0.5}  & \underline{57.9{\tiny$\pm$0.6}} & \textbf{58.1{\tiny$\pm$0.2}} & 87.5{\tiny$\pm$0.5} \\
& ResNet-18   & 47.9{\tiny$\pm$1.8} & 39.6{\tiny$\pm$1.0} & 48.3{\tiny$\pm$1.2} & 50.1{\tiny$\pm$0.5} & 57.1{\tiny$\pm$0.6} & 46.2{\tiny$\pm$0.6} & 48.3{\tiny$\pm$0.8} & 53.7{\tiny$\pm$2.2} & \underline{58.9{\tiny$\pm$1.5}} & 58.2{\tiny$\pm$0.6} & \textbf{59.7{\tiny$\pm$0.4}} & 89.3{\tiny$\pm$1.2} \\
\midrule
\multirow{3}{*}{70~(5.4\%)} 
& ConvNet-6   & 46.3{\tiny$\pm$0.6} & 38.6{\tiny$\pm$0.7} & 46.2{\tiny$\pm$0.6} & 50.1{\tiny$\pm$1.2} & 55.7{\tiny$\pm$0.9} & 47.5{\tiny$\pm$0.8} & 48.9{\tiny$\pm$0.7} & 50.5{\tiny$\pm$0.4} & \underline{55.8{\tiny$\pm$1.7}} & 55.3{\tiny$\pm$0.6} & \textbf{56.4{\tiny$\pm$1.3}} & 86.4{\tiny$\pm$0.2} \\
& ResNetAP-10 & 50.8{\tiny$\pm$0.6} & 45.9{\tiny$\pm$1.5} & 53.4{\tiny$\pm$1.4} & 54.3{\tiny$\pm$0.9} & 58.3{\tiny$\pm$0.2} & 51.7{\tiny$\pm$0.8} & 52.8{\tiny$\pm$1.8} & 54.7{\tiny$\pm$1.6} & \underline{60.6{\tiny$\pm$0.3}} & 60.6{\tiny$\pm$0.1} & \textbf{60.8{\tiny$\pm$0.4}} & 87.5{\tiny$\pm$0.5} \\
& ResNet-18   & 52.1{\tiny$\pm$1.0} & 44.6{\tiny$\pm$1.1} & 49.7{\tiny$\pm$0.8} & 51.5{\tiny$\pm$1.0} & 58.8{\tiny$\pm$0.7} & 51.9{\tiny$\pm$0.8} & 51.1{\tiny$\pm$1.7} & 56.3{\tiny$\pm$1.8} & 60.3{\tiny$\pm$0.3} & \underline{61.0{\tiny$\pm$0.6}} & \textbf{61.2{\tiny$\pm$0.9}} & 89.3{\tiny$\pm$1.2} \\
\midrule
\multirow{3}{*}{100~(7.7\%)} 
& ConvNet-6   & 52.2{\tiny$\pm$0.4} & 45.1{\tiny$\pm$0.5} & 54.4{\tiny$\pm$1.1} & 53.4{\tiny$\pm$0.3} & 61.1{\tiny$\pm$0.7} & 55.0{\tiny$\pm$1.3} & 53.2{\tiny$\pm$0.9} & 57.9{\tiny$\pm$1.5} & \textbf{62.7{\tiny$\pm$1.4}} & 60.6{\tiny$\pm$0.2} & \underline{61.3{\tiny$\pm$1.1}} & 86.4{\tiny$\pm$0.2} \\
& ResNetAP-10 & 59.4{\tiny$\pm$1.0} & 54.8{\tiny$\pm$0.2} & 61.7{\tiny$\pm$0.9} & 58.3{\tiny$\pm$0.8} & 64.5{\tiny$\pm$0.2} & 56.4{\tiny$\pm$0.8} & 56.1{\tiny$\pm$0.9} & 59.5{\tiny$\pm$1.8} & \underline{65.7{\tiny$\pm$0.5}} & 65.2{\tiny$\pm$0.3}  & \textbf{65.8{\tiny$\pm$0.6}} & 87.5{\tiny$\pm$0.5} \\
& ResNet-18   & 61.5{\tiny$\pm$1.3} & 50.4{\tiny$\pm$0.4} & 59.3{\tiny$\pm$0.7} & 58.9{\tiny$\pm$1.3} & 65.7{\tiny$\pm$0.4} & 60.2{\tiny$\pm$1.0} & 58.3{\tiny$\pm$1.2} & 63.8{\tiny$\pm$1.3} & \underline{68.3{\tiny$\pm$0.4}} & 66.3{\tiny$\pm$0.6}  & \textbf{68.4{\tiny$\pm$0.2}} & 89.3{\tiny$\pm$1.2} \\
\bottomrule
\end{tabular}}
\label{tab:imagewoof}
\end{table*}
\vspace{-5pt}


\subsection{Experimental Setup}
\label{subsec:expset}

\noindent\textbf{Generative Backbone.}
For classification experiments, we adopt a pretrained DiT-XL/2~\cite{DiT} as the generative backbone and fine-tune it using DiffFit~\cite{difffit},
all synthetic images are generated at a resolution of $256\times256$.
For detection and segmentation, we follow the experimental protocol of UniDD~\cite{unidd} and employ Stable Diffusion v1.4~\cite{sd14} as the generative backbone to synthesize labeled images with structured annotations.

\noindent\textbf{Implementation Details.}
For classification, student networks are trained from scratch for 300 epochs using SGD with cosine learning-rate decay.
For downstream tasks, we use Faster R-CNN~\cite{ren2016fasterrcnnrealtimeobject} with a ResNet-50 backbone for object detection
and LR-ASPP with a MobileNetV3~\cite{howard2019searchingmobilenetv3} backbone for semantic segmentation, with standard pretrained backbone weights.
All experiments are conducted on a single NVIDIA A100 GPU (80GB).
To ensure statistical reliability, each synthetic dataset is generated three times and each downstream model is trained with three random seeds;
we report mean $\pm$ standard deviation. More details see Appendix~\ref{sec: supply}.

\noindent\textbf{Baselines.}
For classification, we compare against representative optimization-based and diffusion-based dataset distillation methods,
including DM~\cite{DM}, IDC~\cite{IDC}, SRe$^2$L~\cite{Sre2L}, RDED~\cite{RDED}, Minimax~\cite{minimax},
D$^4$M~\cite{D4M},  DDVLCP~\cite{ddvlcp}, CaO2~\cite{cao2}, and ManifoldGD~\cite{roy2026manifoldgd}.
For segmentation and object detection, we primarily compare with UniDD~\cite{unidd},
a recent universal dataset distillation framework. Since diffusion-based distillation for structured outputs remains relatively underexplored,
we additionally include coreset selection baselines (Random, Uniform, K-Center~\cite{k-center}, and Herding~\cite{herding}).

\noindent\textbf{Evaluation Protocol.}
For classification, we evaluate ConvNet-6, ResNetAP-10, and ResNet-18 trained from scratch.
For segmentation and object detection, we follow standard training recipes with pretrained backbones as described above.
Following prior works, we report the better performance between hard-label and soft-label supervision.

\subsection{Results on Classification}
\label{subsec:cls}

We first evaluate our approach on classification benchmarks,
providing comprehensive comparisons with multiple state-of-the-art dataset distillation methods across four benchmark datasets: ImageWoof~\cite{Howard_Imagewoof_2019}, ImageNette~\cite{Howard_Imagenette_2019}, ImageIDC~\cite{IDC} and ImageNet-1K~\cite{imagenet15russakovsky}. The results are summarized in Tables~\ref{tab:imagewoof}, \ref{tab:nette_idc} and \ref{tab:imagenet1k}. Additional comparison experiment results with CaO2~\cite{cao2} are provided in Appendix~\ref{sec: addexp}, Table~\ref{backbone}. 
\begin{table}[H]
\centering
\caption{Comparison on ImageNette and ImageIDC under different IPC settings. Results are obtained on ResNet-18. The best results are marked as bold and the second are underlined.}

\small
\setlength{\tabcolsep}{4pt}
\renewcommand{\arraystretch}{1.0}

\begin{tabular}{c c c c c c c|c}
\toprule
\textbf{} & IPC & Random & DiT & Minimax & D$^4$M & DDVLCP & \mgpo~(Ours) \\
\midrule

\multirow{3}{*}{Nette}
& 10 & 54.2{\tiny$\pm$1.6} & 59.1{\tiny$\pm$0.7} & 58.6{\tiny$\pm$0.4} & 60.9{\tiny$\pm$1.7} & \underline{64.8{\tiny$\pm$3.6}} & \textbf{66.9{\tiny$\pm$2.2}} \\
& 20 & 63.5{\tiny$\pm$0.5} & 64.8{\tiny$\pm$1.2} & 70.6{\tiny$\pm$0.3} & 66.3{\tiny$\pm$1.3} & \underline{71.4{\tiny$\pm$0.5}} & \textbf{73.9{\tiny$\pm$0.7}} \\
& 50 & 76.1{\tiny$\pm$1.1} & 73.3{\tiny$\pm$0.9} & \underline{83.7{\tiny$\pm$0.4}} & 77.7{\tiny$\pm$1.1} & 81.2{\tiny$\pm$0.8} & \textbf{85.4{\tiny$\pm$1.1}} \\

\midrule

\multirow{3}{*}{IDC}
& 10 & 48.1{\tiny$\pm$0.8} & 54.1{\tiny$\pm$0.4} & 51.9{\tiny$\pm$1.4} & 52.8{\tiny$\pm$0.5} & \textbf{57.0{\tiny$\pm$1.4}} & \underline{54.6{\tiny$\pm$1.7}} \\
& 20 & 52.5{\tiny$\pm$0.9} & 58.9{\tiny$\pm$0.2} & 59.1{\tiny$\pm$3.7} & 58.5{\tiny$\pm$0.4} & \underline{63.3{\tiny$\pm$1.2}} & \textbf{64.8{\tiny$\pm$0.4}} \\
& 50 & 68.1{\tiny$\pm$0.7} & 64.3{\tiny$\pm$0.6} & 69.4{\tiny$\pm$1.4} & 69.1{\tiny$\pm$0.8} & \underline{71.9{\tiny$\pm$0.4}} & \textbf{77.8{\tiny$\pm$0.9}} \\

\bottomrule
\end{tabular}

\label{tab:nette_idc}
\end{table}
\vspace{-20pt}
\begin{table}[H]
\centering
\caption{Comparison of dataset distillation methods on ImageNet-1K under different IPC settings.}

\small
\setlength{\tabcolsep}{4pt}
\renewcommand{\arraystretch}{1.0}

\begin{tabular}{ccccc|c}
\toprule
IPC & SRe$^2$L & RDED & DiT & Minimax & \mgpo~(Ours) \\
\midrule
10 & 21.3{\tiny$\pm$0.6} & 42.0{\tiny$\pm$0.1} & 39.6{\tiny$\pm$0.4} & 44.3{\tiny$\pm$0.5} & \textbf{45.1{\tiny$\pm$0.6}} \\
50 & 46.8{\tiny$\pm$0.2} & 56.5{\tiny$\pm$0.1} & 52.9{\tiny$\pm$0.6} & 58.6{\tiny$\pm$0.3} & \textbf{59.1{\tiny$\pm$0.4}} \\
\bottomrule
\end{tabular}
\label{tab:imagenet1k}
\end{table}
\vspace{-10pt}

\subsection{Result on Detection and Segmentation}
\label{subsec:dense}

We further investigate whether the distilled datasets preserve spatial structures that generalize to downstream tasks, particularly semantic segmentation and object detection.

\noindent\textbf{Object Detection.}
Table~\ref{tab:detection_detailed_std} reports object detection results on Pascal VOC 2012~\cite{pascal-voc-2012} and MS-COCO~\cite{lin2015microsoftcococommonobjects}.
Under aggressive compression, our method substantially outperforms coreset selection baselines,
which indicates that the distilled images retain object-level localization cues rather than merely category-level appearance statistics.

\noindent\textbf{Segmentation.}
Table~\ref{tab:seclect_ablation} summarizes segmentation performance on Pascal VOC~\cite{pascal-voc-2012}.
Our method achieves consistently higher mIoU across compression ratios,
suggesting that pixel-level structural information is preserved during distillation and remains useful for downstream learning.
\begin{table*}[h]
\centering
\caption{Performance comparison with coreset selection methods. Metrics are reported as mAP and AP$_{50}$(\%) across different selection ratios.}
\label{tab:detection_detailed_std}

\resizebox{\linewidth}{!}{
\begin{tabular}{l cc cc cc cc cc cc}
\toprule
\multirow{3}{*}{Method} & \multicolumn{6}{c}{Pascal VOC} & \multicolumn{6}{c}{MS COCO} \\
\cmidrule(lr){2-7} \cmidrule(lr){8-13}
& \multicolumn{2}{c}{0.5\%} & \multicolumn{2}{c}{1\%} & \multicolumn{2}{c}{2\%} & \multicolumn{2}{c}{0.25\%} & \multicolumn{2}{c}{0.5\%} & \multicolumn{2}{c}{1\%} \\
\cmidrule(lr){2-3} \cmidrule(lr){4-5} \cmidrule(lr){6-7} \cmidrule(lr){8-9} \cmidrule(lr){10-11} \cmidrule(lr){12-13}
& mAP & AP$_{50}$ & mAP & AP$_{50}$ & mAP & AP$_{50}$ & mAP & AP$_{50}$ & mAP & AP$_{50}$ & mAP & AP$_{50}$ \\
\midrule
Random   & 0.8{\tiny $\pm$0.2} & 3.1{\tiny $\pm$0.4} & 4.2{\tiny $\pm$0.5} & 13.7{\tiny $\pm$0.6} & 12.4{\tiny $\pm$0.4} & 34.3{\tiny $\pm$0.5} & 0.5{\tiny $\pm$0.1} & 1.7{\tiny $\pm$0.3} & 3.7{\tiny $\pm$0.2} & 10.1{\tiny $\pm$0.3} & 7.2{\tiny $\pm$0.8} & 17.3{\tiny $\pm$0.9} \\
Uniform  & 0.9{\tiny $\pm$0.1} & 3.4{\tiny $\pm$0.3} & 5.7{\tiny $\pm$0.2} & 17.7{\tiny $\pm$0.4} & 13.8{\tiny $\pm$0.3} & 36.2{\tiny $\pm$0.4} & 0.8{\tiny $\pm$0.2} & 2.4{\tiny $\pm$0.5} & 3.4{\tiny $\pm$0.4} & 9.5{\tiny $\pm$0.6}  & 7.4{\tiny $\pm$0.5} & 17.6{\tiny $\pm$0.5} \\
K-Center & 0.5{\tiny $\pm$0.1} & 2.1{\tiny $\pm$0.3} & 3.6{\tiny $\pm$0.6} & 12.3{\tiny $\pm$0.3} & 10.9{\tiny $\pm$0.6} & 29.3{\tiny $\pm$0.6} & 0.4{\tiny $\pm$0.1} & 1.5{\tiny $\pm$0.2} & 3.2{\tiny $\pm$0.5} & 9.5{\tiny $\pm$0.5}  & 6.1{\tiny $\pm$0.3} & 15.4{\tiny $\pm$0.6} \\
Herding  & 0.6{\tiny $\pm$0.2} & 2.4{\tiny $\pm$0.2} & 3.5{\tiny $\pm$0.5} & 11.9{\tiny $\pm$0.5} & 10.4{\tiny $\pm$0.4} & 28.7{\tiny $\pm$0.7} & 0.5{\tiny $\pm$0.1} & 1.8{\tiny $\pm$0.4} & 3.5{\tiny $\pm$0.3} & 9.7{\tiny $\pm$0.3}  & 6.7{\tiny $\pm$0.4} & 16.3{\tiny $\pm$0.7} \\
UniDD    & 8.5{\tiny $\pm$0.4} & 22.3{\tiny $\pm$0.6} & 16.8{\tiny $\pm$0.5} & 38.9{\tiny $\pm$0.7} & 23.9{\tiny $\pm$0.5} & 48.5{\tiny $\pm$0.6} & 4.5{\tiny $\pm$0.3} & 10.3{\tiny $\pm$0.4} & 7.1{\tiny $\pm$0.4} & 16.9{\tiny $\pm$0.3} & 10.8{\tiny $\pm$0.4} & 22.5{\tiny $\pm$0.5} \\
\rowcolor{tabhighlight}
\textbf{\mgpo (Ours)} & \textbf{13.4}{\tiny $\pm$0.2} & \textbf{26.3}{\tiny $\pm$0.5} & \textbf{21.1}{\tiny $\pm$0.4} & \textbf{41.8}{\tiny $\pm$0.3} & \textbf{26.9}{\tiny $\pm$0.9} & \textbf{53.1}{\tiny $\pm$1.3} & \textbf{5.8}{\tiny $\pm$0.2} & \textbf{13.9}{\tiny $\pm$0.3} & \textbf{8.3}{\tiny $\pm$0.5} & \textbf{19.9}{\tiny $\pm$0.2} & \textbf{11.7}{\tiny $\pm$0.8} & \textbf{25.4}{\tiny $\pm$0.8} \\
\midrule
Full Data & \multicolumn{6}{c}{51.4$\pm$0.8 / 80.3$\pm$0.4} & \multicolumn{6}{c}{32.6$\pm$0.7 / 51.4$\pm$0.8} \\
\bottomrule
\end{tabular}
}
\end{table*}
\vspace{-20pt}
\begin{table}[H]
\centering

\begin{minipage}[c]{0.48\linewidth}
\centering
\caption{Semantic Segmentation results.}
\label{tab:seclect_ablation}
\small
\setlength{\tabcolsep}{3pt}
\renewcommand{\arraystretch}{1.0}

\resizebox{\linewidth}{!}{
\begin{tabular}{c|cccc}
\toprule
\multirow{2}{*}{Methods} & \multicolumn{4}{c}{Semantic Segmentation (mIoU)} \\
\cmidrule(lr){2-5}
 & 3.5\% & 7\% &  14\% & Full \\
\midrule
Random   & 6.7$\pm$0.4  & 9.6$\pm$0.5  & 11.3$\pm$0.4 &  \\
Uniform  & 6.5$\pm$0.6  & 9.3$\pm$0.4  & 12.1$\pm$0.5 &  \\
K-Center & 6.1$\pm$0.4  & 8.5$\pm$0.6  & 10.4$\pm$0.7 & \multirow{2}{*}{32.9$\pm$0.6} \\
Herding  & 6.2$\pm$0.7  & 8.7$\pm$0.4  & 10.6$\pm$0.3 &  \\
UniDD & 10.3$\pm$0.5 & 12.3$\pm$0.3 & 14.9$\pm$0.7 &  \\
\rowcolor{tabhighlight}
\mgpo~(Ours) & \textbf{18.3$\pm$0.5} & \textbf{20.4$\pm$0.4} & \textbf{23.1$\pm$0.6} \\

\bottomrule
\end{tabular}
}
\end{minipage}
\hfill
\begin{minipage}[c]{0.48\linewidth}
\centering
\caption{Ablation study of rewards.}
\small
\setlength{\tabcolsep}{3pt}
\renewcommand{\arraystretch}{1.1}

\resizebox{\linewidth}{!}{
\begin{tabular}{l|cc|cc}
\toprule
\multirow{2}{*}{Method} & 
\multicolumn{2}{c|}{ImageWoof} & 
\multicolumn{2}{c}{ImageNette} \\
\cmidrule(lr){2-3} \cmidrule(lr){4-5}
 & IPC = 10 & IPC = 50 & IPC = 10 & IPC = 50 \\
\midrule
DiT         & 34.7$\pm$0.5       & 49.3$\pm$0.2       & 58.9$\pm$1.2  & 82.9$\pm$0.1 \\
Pixel Reward        & 37.0 $\pm$0.3     & 50.5$\pm$0.5      & 64.1$\pm$0.9   & 84.9$\pm$1.1 \\
Latent Reward      & 36.8$\pm$0.3   & 50.0$\pm$0.2  & 65.9$\pm$0.2 & 84.9$\pm$1.3 \\
\midrule
\rowcolor{tabhighlight}
Both (Ours)      & \textbf{40.9$\pm$1.1}      & \textbf{57.9$\pm$0.2}      & \textbf{66.9$\pm$2.2}  & \textbf{85.4$\pm$1.1} \\
\bottomrule
\end{tabular}}
\label{abl2}
\end{minipage}
\vspace{-2.0mm}
\end{table}

\subsection{Ablation Study}
\label{sec:abl}
We conduct ablation studies to understand how each reward component contributes to both semantic fidelity and spatial consistency.
Specifically, we progressively enable the baseline DiT objective, Pixel Reward, MST Reward, and their combination,
with results reported in Table~\ref{abl2}. In general, when both rewards are applied, best performance is observed across different IPC settings.

We further analyze robustness across training hyperparameters 
(Figure~\ref{fig:ablation_2x2}) and finetuning design choices 
(Figure~\ref{fig:finetune_param}).
\mgpo~remains stable across all settings, yielding the best trade-off.

\begin{figure}[H]
\captionsetup{skip=3pt}
\centering
\begin{minipage}{0.499\linewidth}
\centering
\includegraphics[width=\linewidth, height=25cm, keepaspectratio]{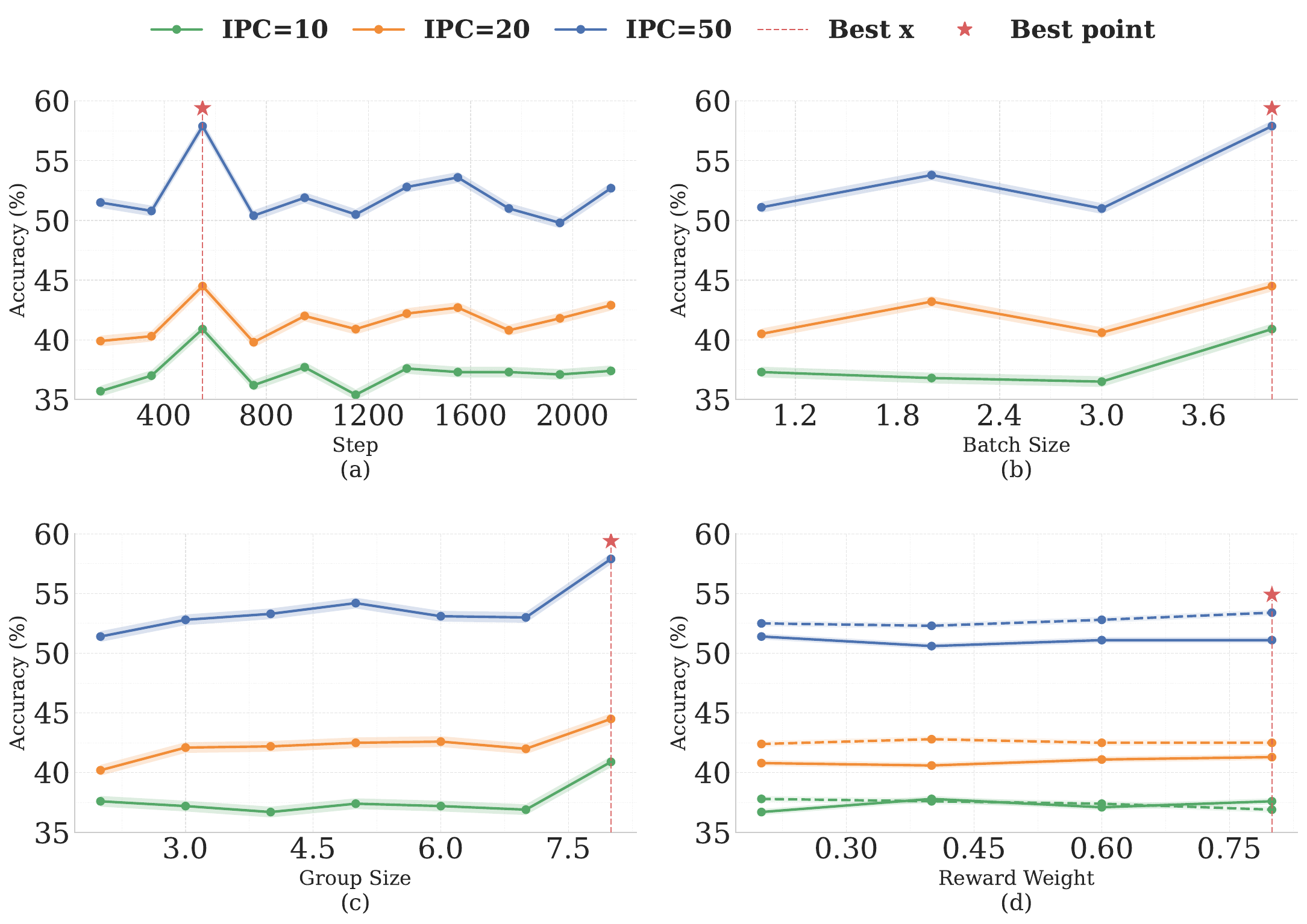}
\caption{Hyperparameter analysis of \mgpo.}
\label{fig:ablation_2x2}
\end{minipage}
\hfill
\begin{minipage}{0.495\linewidth}
\centering
\includegraphics[width=\linewidth, height=25cm, keepaspectratio]{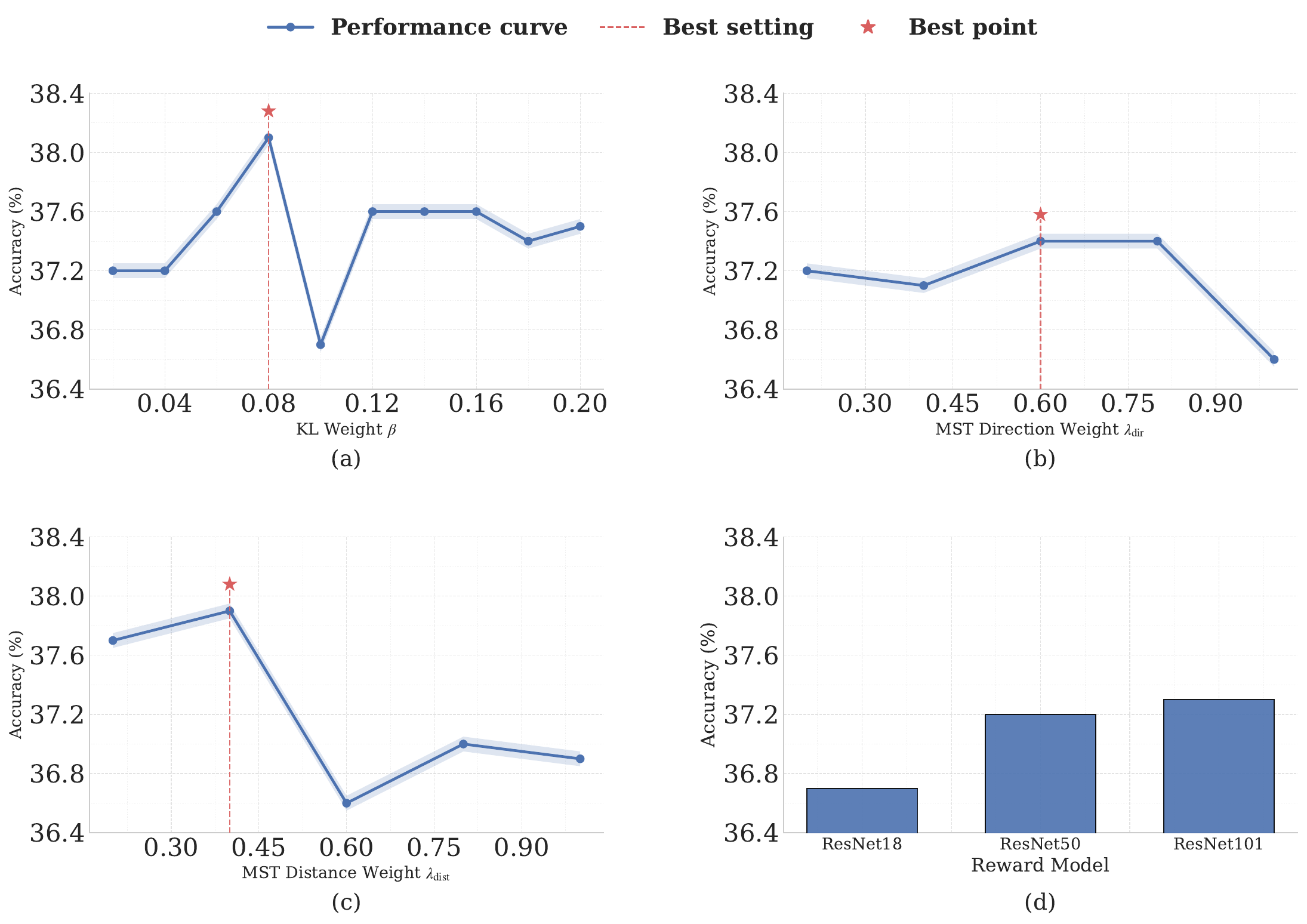}
\caption{Ablation studies on key design choices of \mgpo~ finetuning process.}
\label{fig:finetune_param}
\end{minipage}
\end{figure}
\vspace{-10pt}

\section{Conclusion}
\label{sec:conclusion}

In this paper, we proposed~\mgpo~, a reward driven RL framework that resolves the fundamental objective mismatch in diffusion-based dataset distillation. By integrating pixel level discriminative and latent level geometric rewards,~\mgpo~effectively aligns the generative process with downstream tasks. Furthermore, our plug-and-play reward design enables the framework
to extend beyond standard classification to complex downstream tasks like object detection and semantic segmentation. Extensive experiments show that~\mgpo~achieves state of the art performance across multiple classification benchmarks and successfully reaches state of the art on detection and segmentation tasks.

\subsection*{AI Use Statement}

In this work, we used generative AI tools (large language models) for the
following tasks with required disclosure: \emph{proof assistance}, namely
helping to derive and write the proofs of the theoretical results in
Appendix~\ref{proof:mstasms} and Appendix~\ref{proof:subsample-mst}
(Lemmas~\ref{lem:one-sided}--\ref{lem:covering-rate},
Theorems~\ref{thm:mst} and~\ref{thm:subsample-mst},
Proposition~\ref{prop:node-deployed}, and the associated corollaries and remarks),
and \emph{drafting} parts of the manuscript, primarily the related-work
discussion (Section~2 and Appendix~\ref{sec: related}) and
the implementation descriptions in Appendix~\ref{sec: supply}.
We have not used generative AI tools for research ideation, hypothesis
refinement, methodology design, implementation of the method, running
experiments, data cleaning, or result interpretation; the remaining required
disclosure tasks (e.g., translation and qualitative analysis) are not applicable
to this work. Note that the diffusion models used to synthesize distilled
datasets are the object of study and are not AI assistance in preparing this
paper. Additionally, we used generative AI tools to aid and polish the writing
(grammar, wording, and clarity) throughout the paper.

We have reviewed all AI-assisted work. Every AI-assisted proof was checked
step by step by at least two authors, who verified all assumptions,
constants, and cited results against the original sources
(e.g., \citet{berenfeld2022estimating}, \citet{niyogi2008finding}).
All AI-drafted or AI-polished text was revised by the authors, and every
citation was checked manually against the referenced paper. All numerical
results were produced by the authors' own code and were not generated or
edited by AI tools. We take responsibility for the final content of this work,
including text, claims, or artifacts produced with the aid of generative AI.

\subsection*{Ethics Statement}

This work does not involve human subjects, user studies, or the collection of
new data. All experiments use publicly available benchmarks (ImageNet-1K and
its subsets ImageWoof, ImageNette, and ImageIDC; PASCAL VOC 2012; MS-COCO),
used under their respective licenses and terms of use; we do not redistribute
the original images. Some of these datasets contain images of people collected
from the web, and the synthesized distilled datasets may therefore reflect
content and biases present in the source data. In addition, because \mgpo{}
optimizes the generator with frozen pretrained reward models (classifier,
detector, or segmentation network), any biases in these reward models may be
internalized and amplified in the distilled data; we discuss this risk and
recommend auditing reward models before deployment in
Appendix~\ref{sec:limitations_impacts}. We do not foresee direct harmful
applications beyond those common to generative models and dataset
compression. The authors declare no conflicts of interest or sponsorship that
could have influenced this work.

\subsection*{Reproducibility Statement}

We have made the following efforts to ensure reproducibility. The full
method is described in Section~\ref{sec:method}, including the DDIM
sampling formulation (Section~\ref{subsec: DDIM}), the reward definitions
(Section~\ref{subsec:rewards}), and the complete training procedure
(Algorithm~\ref{alg:mgpo}). The MST construction is given in
Algorithm~\ref{alg:mst}, with all preprocessing steps and hyperparameters
(e.g., $k$-NN size, PCA dimension, number of sampled latents per class)
detailed in Appendix~\ref{sec:supp_mst_details}; the latent geometric reward
and the pixel-space rewards for classification, detection, and segmentation
are specified in Appendices~\ref{sec:supp_latent_reward}
and~\ref{sec:supp_pixel_reward} (Algorithms~\ref{alg:det_reward}
and~\ref{alg:seg_reward}). Experimental settings, including backbones,
student architectures, training schedules, evaluation protocols, and the
number of runs and random seeds, are given in Section~\ref{subsec:expset},
and hyperparameter sensitivity analyses are reported in
Section~\ref{sec:abl} and Appendix~\ref{sec: addexp}. For the theoretical
results, all assumptions are stated explicitly and complete proofs are
provided in Appendices~\ref{sec:makchain}, \ref{proof:mstasms},
and~\ref{proof:subsample-mst}. All datasets are publicly available.
The source code will be released upon acceptance.

\bibliography{nips}
\bibliographystyle{iclr2027_conference}

\newpage
\begin{center}
{\LARGE\sc Appendix\par}
\end{center}
\vskip 0.3in
\appendix

\startcontents[appendix]
\printcontents[appendix]{}{0}{\section*{Contents of Appendix}\setcounter{tocdepth}{2}}
\vspace{1em}

\section{Theoretical Analysis}
\subsection{Reward-Aligned Diffusion as Policy Optimization}
\label{sec:makchain}
Let \(p_\theta(x_0)\) denote the data distribution induced by a the current diffusion model, and let \(p_{\mathrm{ref}}(x_0)\) denote the distribution of the frozen pretrained model
Standard diffusion training maximizes likelihood and therefore prioritizes perceptual fidelity.
However, in dataset distillation, our objective is not merely realism,
but the generation of samples that are informative for downstream tasks.

To incorporate task objectives, we consider a reward-reweighted target distribution
\begin{equation}
p_\beta(x_0)
\;\propto\;
p_\mathrm{ref}(x_0)\exp\!\big(\beta R(x_0)\big),
\label{eq:reward_reweighted_density_new}
\end{equation}
where \(R(x_0)\) denotes a task-specific reward and \(\beta>0\) controls alignment strength.
This formulation shifts probability mass toward task-relevant regions of the data manifold.

We view diffusion sampling process as a trajectory
\(
\tau = \{x_T, x_{T-1}, \dots, x_0\}
\)
induced by a parameterized policy \(p_\theta(\tau)\).

Under this perspective, aligning diffusion generation with task objectives
amounts to optimizing the expected terminal reward over trajectories, regularized towards the pretrained model:
\begin{equation}
\mathcal{J}_\beta(\theta)
=
\mathbb{E}_{\tau \sim p_\theta}
\big[
R(x_0)
\big]
\;-\;\tfrac{1}{\beta}\,\mathrm{KL}\big(p_\theta(x_0)\,\|\,p_{\mathrm{ref}}(x_0)\big).
\label{eq:trajectory_objective}
\end{equation}

Over all distributions of $x_0$, the maximizer of Eq.~(\ref{eq:trajectory_objective}) is exactly the reward-reweighted density in Eq.~(\ref{eq:reward_reweighted_density_new}), since $\mathcal{J}_\beta = \frac{1}{\beta}\log Z_\beta - \frac{1}{\beta}\mathrm{KL}(p_\theta \,\|\, p_\beta)$ with $Z_\beta = \mathbb{E}_{p_{\mathrm{ref}}}[\exp(\beta R)]$. Policy optimization over the diffusion parameters therefore approximates reward-aligned sampling within the model class, without modifying the inference procedure.

This policy view naturally enables the incorporation of
heterogeneous reward signals.
In our setting, the reward function will consist of
(i) a discriminative term defined in pixel space
and (ii) a geometric term defined in latent space,
allowing alignment of both semantic and structural properties.

\subsection{Theoretical Justification of MST as a Manifold Skeleton}
\label{proof:mstasms}

\paragraph{Setup.}
Let $\manifold \subset \R^d$ be a compact, connected, $C^2$-embedded submanifold 
with reach
\begin{equation*}
    \tau \defeq \reach(\manifold) > 0.
\end{equation*}
Let $U = \{u_1, \ldots, u_N\} \subset \manifold$ be a finite point set on the 
manifold, and define its \emph{covering radius}
\begin{equation}
    \delta \defeq \sup_{x \in \manifold}\, \min_{1 \leq i \leq N} \|x - u_i\|_2.
    \label{eq:def-covering-radius}
\end{equation}
Let $\mathcal{T}_{\mathrm{MST}} = (U, E_{\mathrm{MST}})$ denote the Euclidean 
minimum spanning tree on $U$, and write
\begin{equation*}
    |\mathcal{T}_{\mathrm{MST}}| \defeq 
    \bigcup_{(u_i, u_j) \in E_{\mathrm{MST}}} [u_i, u_j] \;\subset\; \R^d
\end{equation*}
for its geometric realization, where $[u_i, u_j]$ denotes the closed line 
segment in $\R^d$ between $u_i$ and $u_j$.

\textbf{Scope.} The analysis in this subsection and in Appendix~\ref{proof:subsample-mst} is idealized: it concerns a Euclidean MST built on points sampled from a smooth class manifold, with the distance to the tree realization $|\mathcal{T}_{\mathrm{MST}}|$ (i.e., to the nearest edge segment). It motivates the geometric reward but does not by itself cover the implemented reward, which uses the nearest-node distance in the VAE/PCA representation, a $k$NN-restricted MST, and a direction term. \cref{prop:node-deployed} below states what carries over to the implementation.

The goal of this subsection is to establish that, under the regime 
$\delta < \tau$, the MST realization $|\mathcal{T}_{\mathrm{MST}}|$ approximates 
the manifold $\manifold$ in the Hausdorff metric at rate $\delta$ in the 
forward direction and rate $\delta^2/\tau$ in the backward direction. 
The argument proceeds via four lemmas: 
(\textit{i}) a one-sided Hausdorff bound from the manifold to any spanning tree 
(\cref{lem:one-sided}); 
(\textit{ii}) the existence of a short-edged spanning tree on $U$ 
(\cref{lem:local-tree}); 
(\textit{iii}) a bottleneck property of the MST 
(\cref{lem:bottleneck}); 
(\textit{iv}) a chord-to-manifold distance bound based on convexity defect 
(\cref{lem:short-chord}). 
These ingredients are then combined in \cref{thm:mst}.

\begin{lemma}[One-sided Hausdorff bound]
\label{lem:one-sided}
For any spanning tree $\mathcal{T}$ on the vertex set $U$,
\begin{equation*}
    \sup_{x \in \manifold}\, d(x, |\mathcal{T}|) \;\leq\; \delta.
\end{equation*}
\end{lemma}

\begin{proof}
Fix $x \in \manifold$. By the definition of $\delta$ in 
\cref{eq:def-covering-radius}, there exists $u_i \in U$ such that
\begin{equation*}
    \|x - u_i\|_2 \leq \delta.
\end{equation*}
Since $u_i$ is a vertex of $\mathcal{T}$, we have $u_i \in |\mathcal{T}|$. Therefore
\begin{equation*}
    d(x, |\mathcal{T}|) \;\leq\; \|x - u_i\|_2 \;\leq\; \delta.
\end{equation*}
Taking the supremum over $x \in \manifold$ proves the claim.
\end{proof}

\begin{lemma}[Existence of a short-edged spanning tree]
\label{lem:local-tree}
There exists a spanning tree on $U$ whose every edge has Euclidean length at 
most $2\delta$.
\end{lemma}

\begin{proof}
For each $i \in \{1, \ldots, N\}$, define the closed manifold cap
\begin{equation*}
    A_i \defeq \overline{B}(u_i, \delta) \cap \manifold.
\end{equation*}
By the definition of $\delta$, every $x \in \manifold$ lies within distance 
$\delta$ of some $u_i$, hence
\begin{equation*}
    \manifold \;=\; \bigcup_{i=1}^{N} A_i.
\end{equation*}

\emph{Claim: the intersection graph of $\{A_i\}_{i=1}^N$ is connected.} 
Suppose, for contradiction, that the intersection graph is disconnected, so 
there is a partition
\begin{equation*}
    \{1, \ldots, N\} = I \sqcup J, \qquad I, J \neq \varnothing,
\end{equation*}
satisfying
\begin{equation*}
    A_i \cap A_j = \varnothing \qquad \text{for all } i \in I,\ j \in J.
\end{equation*}
Define
\begin{equation*}
    X \defeq \bigcup_{i \in I} A_i, \qquad Y \defeq \bigcup_{j \in J} A_j.
\end{equation*}
Each $A_i$ is closed in $\manifold$ (as the intersection of a closed Euclidean 
ball with $\manifold$), so $X$ and $Y$ are closed in $\manifold$. They are 
nonempty, disjoint, and satisfy $\manifold = X \cup Y$. In particular, 
$X = \manifold \setminus Y$ is also open in $\manifold$, so $X$ is a nonempty 
proper clopen subset of $\manifold$. This contradicts the connectedness of 
$\manifold$, establishing the claim.

Now construct a graph $G$ on vertex set $U$ by joining $u_i$ and $u_j$ whenever 
$A_i \cap A_j \neq \varnothing$. By the claim, $G$ is connected. For any such 
edge $(u_i, u_j)$, choose $x \in A_i \cap A_j$; the triangle inequality yields
\begin{equation*}
    \|u_i - u_j\|_2 \;\leq\; \|u_i - x\|_2 + \|x - u_j\|_2 \;\leq\; 2\delta.
\end{equation*}
Thus every edge of $G$ has length at most $2\delta$, and any spanning tree of 
$G$ inherits this property.
\end{proof}

\begin{lemma}[Bottleneck bound on MST edge length]
\label{lem:bottleneck}
Let
\begin{equation*}
    \ell_{\mathrm{MST}} \defeq \max_{(u_i, u_j) \in E_{\mathrm{MST}}} \|u_i - u_j\|_2.
\end{equation*}
Then $\ell_{\mathrm{MST}} \leq 2\delta$.
\end{lemma}

\begin{proof}
By Lemma~\ref{lem:local-tree}, there exists a spanning tree on $U$ whose maximum 
edge length is at most $2\delta$. Every Euclidean minimum spanning tree is 
also a minimum bottleneck spanning tree, i.e., it minimizes the maximum edge 
length among all spanning trees. Hence
\begin{equation*}
    \ell_{\mathrm{MST}} \;\leq\; 2\delta. \qedhere
\end{equation*}
\end{proof}

\begin{lemma}[Chord distance to manifold via convexity defect]
\label{lem:short-chord}
Let $\manifold \subset \R^d$ have reach $\tau > 0$. Suppose $a, b \in \manifold$ 
satisfy
\begin{equation*}
    \|a - b\|_2 = \ell < 2\tau.
\end{equation*}
Then for every $y \in [a, b]$,
\begin{equation}
    d(y, \manifold) \;\leq\; \tau - \sqrt{\tau^2 - \tfrac{\ell^2}{4}} 
    \;\leq\; \frac{\ell^2}{4\tau}.
    \label{eq:chord-bound}
\end{equation}
\end{lemma}

\begin{proof}
Let $t \defeq \ell/2 < \tau$. The smallest closed ball in $\R^d$ containing 
$\{a, b\}$ has radius $t$, namely the ball centered at the midpoint of 
$[a, b]$. Since
\begin{equation*}
    [a, b] = \mathrm{conv}\{a, b\},
\end{equation*}
every $y \in [a, b]$ belongs to the $t$-local convex hull of $\manifold$ in 
the sense of~\cite[Definition~4.1]{berenfeld2022estimating}. Therefore
\begin{equation*}
    d(y, \manifold) \;\leq\; h_{\manifold}(t),
\end{equation*}
where $h_{\manifold}$ denotes the convexity defect function of $\manifold$.

For a set of reach $\tau$, the quarter-circle bound 
of~\cite[Definition~4.1 and the bound stated immediately thereafter]{berenfeld2022estimating} 
states
\begin{equation*}
    h_{\manifold}(t) \;\leq\; \tau - \sqrt{\tau^2 - t^2}, 
    \qquad 0 \leq t < \tau.
\end{equation*}
Substituting $t = \ell/2$ yields the first inequality of~\cref{eq:chord-bound}.

For the second inequality, observe that for any $0 \leq s \leq \tau^2$,
\begin{equation*}
    \tau - \sqrt{\tau^2 - s}
    \;=\; \frac{s}{\tau + \sqrt{\tau^2 - s}}
    \;\leq\; \frac{s}{\tau},
\end{equation*}
where the equality follows from rationalizing the numerator and the inequality 
from $\sqrt{\tau^2 - s} \geq 0$. Setting $s = \ell^2/4$ gives
\begin{equation*}
    \tau - \sqrt{\tau^2 - \tfrac{\ell^2}{4}} \;\leq\; \frac{\ell^2}{4\tau},
\end{equation*}
which completes the proof.
\end{proof}

We are now ready to state and prove the main result of this subsection.

\begin{theorem}[MST--manifold Hausdorff bound]
\label{thm:mst}
Suppose $\delta < \tau$. Then
\begin{equation}
    d_H\!\bigl(\manifold,\ |\mathcal{T}_{\mathrm{MST}}|\bigr) \;\leq\; \delta.
    \label{eq:mst-hausdorff}
\end{equation}
More precisely, the two one-sided bounds
\begin{align}
    \sup_{x \in \manifold} d\!\bigl(x,\ |\mathcal{T}_{\mathrm{MST}}|\bigr) 
    &\;\leq\; \delta, 
    \label{eq:forward-bound} \\
    \sup_{y \in |\mathcal{T}_{\mathrm{MST}}|} d\bigl(y,\ \manifold\bigr) 
    &\;\leq\; \tau - \sqrt{\tau^2 - \delta^2} \;\leq\; \frac{\delta^2}{\tau}
    \label{eq:backward-bound}
\end{align}
hold simultaneously.
\end{theorem}

\begin{proof}
The forward bound~\cref{eq:forward-bound} is exactly \cref{lem:one-sided}.

For the backward bound~\cref{eq:backward-bound}, by \cref{lem:bottleneck},
\begin{equation*}
    \ell_{\mathrm{MST}} \;\leq\; 2\delta \;<\; 2\tau,
\end{equation*}
so \cref{lem:short-chord} applies to every edge $[u_i, u_j] \in 
E_{\mathrm{MST}}$. Each point $y$ on such an edge therefore satisfies
\begin{equation*}
    d(y, \manifold) \;\leq\; \tau - \sqrt{\tau^2 - \tfrac{\|u_i - u_j\|_2^2}{4}}.
\end{equation*}
Since the function $\ell \mapsto \tau - \sqrt{\tau^2 - \ell^2/4}$ is monotone 
non-decreasing on $[0, 2\tau)$, and $\|u_i - u_j\|_2 \leq \ell_{\mathrm{MST}} 
\leq 2\delta$, we have
\begin{equation*}
    d(y, \manifold) 
    \;\leq\; \tau - \sqrt{\tau^2 - \tfrac{\ell_{\mathrm{MST}}^2}{4}}
    \;\leq\; \tau - \sqrt{\tau^2 - \delta^2}
    \;\leq\; \frac{\delta^2}{\tau},
\end{equation*}
where the last inequality follows from \cref{lem:short-chord} applied with 
$s = \delta^2$. Taking the supremum over $y \in |\mathcal{T}_{\mathrm{MST}}|$ 
yields~\cref{eq:backward-bound}.

To deduce the combined Hausdorff bound~\cref{eq:mst-hausdorff}, note that
\begin{equation*}
    d_H\!\bigl(\manifold,\ |\mathcal{T}_{\mathrm{MST}}|\bigr)
    \;=\;\max\!\left\{
        \sup_{x \in \manifold} d\!\bigl(x, |\mathcal{T}_{\mathrm{MST}}|\bigr),\ 
        \sup_{y \in |\mathcal{T}_{\mathrm{MST}}|} d(y, \manifold)
    \right\}.
\end{equation*}
By \cref{eq:forward-bound,eq:backward-bound}, both terms inside the maximum are 
at most
\begin{equation*}
    \max\!\left\{\delta,\ \tau - \sqrt{\tau^2 - \delta^2}\right\}.
\end{equation*}
It remains to show $\tau - \sqrt{\tau^2 - \delta^2} \leq \delta$ whenever 
$\delta < \tau$. Equivalently, $\tau - \delta \leq \sqrt{\tau^2 - \delta^2}$, 
which after squaring becomes
\begin{equation*}
    (\tau - \delta)^2 \;\leq\; \tau^2 - \delta^2,
\end{equation*}
i.e., $-2\tau\delta + 2\delta^2 \leq 0$, i.e., $\delta \leq \tau$, which holds 
by assumption. Hence
\begin{equation*}
    d_H\!\bigl(\manifold,\ |\mathcal{T}_{\mathrm{MST}}|\bigr) \;\leq\; \delta. \qedhere
\end{equation*}
\end{proof}

\begin{remark}
\label{rem:asymmetric-bounds}
The bounds in \cref{thm:mst} are \emph{asymmetric}: the forward direction 
(manifold to MST) decays linearly in $\delta$, while the backward direction 
(MST to manifold) decays quadratically as $\delta^2/\tau$. The asymmetry arises 
because the MST realization includes line segments that may bow off the 
manifold, but the bowing distance is controlled by the manifold's reach via 
\cref{lem:short-chord}. This quadratic backward rate is what enables 
\cref{lem:bottleneck} to feed into a tight bound: a generic spanning tree on 
$U$ would only give a linear $\mathcal{O}(\delta)$ bound in both directions.
\end{remark}

\begin{remark}[Geometric approximation, not topology preservation]
\label{rem:no-topology}
\Cref{thm:mst} is a statement about Hausdorff closeness only. The MST is a one-dimensional acyclic graph, so $|\mathcal{T}_{\mathrm{MST}}|$ has trivial first homology and cannot, in general, preserve cycles, higher-dimensional homology, or the intrinsic dimension of $\manifold$. For example, for points sampled on a circle the MST is a path, which breaks the cycle. We therefore use the MST as a sparse geometric skeleton that encourages coverage of the class support, not as a topology-preserving representation.
\end{remark}

\paragraph{What carries over to the implemented reward.}
The proofs of \cref{lem:one-sided,lem:local-tree,lem:bottleneck} use only that $\manifold$ is compact and connected; smoothness and reach enter only through \cref{lem:short-chord}, i.e., the backward bound~\cref{eq:backward-bound}. This allows us to state a bound for the nearest-node distance used in practice, directly in the deployed representation.

\begin{proposition}[Nearest-node distance in the deployed representation]
\label{prop:node-deployed}
Let $K \subset \R^D$ be compact and connected, let $U = \{u_1,\ldots,u_N\} \subset K$ have covering radius $\delta_K \defeq \sup_{x \in K}\min_i \|x - u_i\|_2$, and let $\mathcal{T}$ be the Euclidean MST on $U$. Define the node distance $D_{\mathrm{node}}(y) \defeq \min_i \|y - u_i\|_2$ and the edge distance $D_{\mathrm{edge}}(y) \defeq d(y, |\mathcal{T}|)$. Then:
\begin{enumerate}
    \item[(i)] $\sup_{x \in K} D_{\mathrm{node}}(x) \leq \delta_K$;
    \item[(ii)] for every $y \in \R^D$, \; $D_{\mathrm{edge}}(y) \;\leq\; D_{\mathrm{node}}(y) \;\leq\; D_{\mathrm{edge}}(y) + \delta_K$;
    \item[(iii)] if $V \subseteq U$, then for every $y \in \R^D$, \; $0 \leq D^{V}_{\mathrm{node}}(y) - D^{U}_{\mathrm{node}}(y) \leq \max_{u \in U}\min_{v \in V}\|u - v\|_2 \leq \delta_K(V)$, where $\delta_K(V)$ is the covering radius of $V$ over $K$.
\end{enumerate}
In particular, with $K = g_c(\mathcal{X}_c)$ the image of a compact, connected class support $\mathcal{X}_c$ under the continuous encoder-and-whitening map $g_c$ (Appendix~\ref{sec:supp_mst_details}), (i)--(iii) hold in the VAE/PCA space where the reward is computed, without requiring $g_c$ to be isometric or invertible.
\end{proposition}

\begin{proof}
(i) is the definition of $\delta_K$. For (ii), the left inequality holds because every node lies on $|\mathcal{T}|$. For the right inequality, let $p \in [u_i, u_j]$ be a point of $|\mathcal{T}|$ nearest to $y$. The nearer endpoint of this edge is within $\|u_i - u_j\|_2/2 \le \ell_{\mathrm{MST}}/2$ of $p$, and the arguments of \cref{lem:local-tree,lem:bottleneck}, which use only compactness and connectedness of $K$, give $\ell_{\mathrm{MST}} \le 2\delta_K$. Hence $D_{\mathrm{node}}(y) \le \|y - p\|_2 + \delta_K = D_{\mathrm{edge}}(y) + \delta_K$. For (iii), the left inequality holds since $V \subseteq U$. For the right one, let $u^\ast \in U$ be nearest to $y$ and $v^\ast \in V$ nearest to $u^\ast$; then $D^{V}_{\mathrm{node}}(y) \le \|y - u^\ast\|_2 + \|u^\ast - v^\ast\|_2 \le D^{U}_{\mathrm{node}}(y) + \max_{u \in U}\min_{v\in V}\|u - v\|_2$, and the last quantity is at most $\delta_K(V)$ because $U \subset K$. Finally, the continuous image of a compact connected set is compact and connected.
\end{proof}

\begin{remark}[What remains empirical]
\label{rem:remains-empirical}
\Cref{prop:node-deployed} does not assume smoothness or positive reach of $K$, so the backward bound~\cref{eq:backward-bound} is not claimed in the deployed space. It also does not cover the direction term, which is evaluated empirically, or the $k$NN restriction used to build the MST: if the symmetrized $k$NN graph contains a Euclidean MST of $U$, then its MST is a Euclidean MST of $U$, and otherwise the two trees can differ. Finally, the relation between distances in $g_c(\mathcal{X}_c)$ and downstream utility is not established theoretically and is assessed through the experiments.
\end{remark}

\subsection{Subsample-MST Approximation and Scalability}
\label{proof:subsample-mst}

\paragraph{Setup.}
We adopt the notation of~\cref{proof:mstasms}. Let $\mathcal M_c \subset \mathbb R^d$ 
denote the class-$c$ data manifold, assumed to be a compact, connected, 
$C^2$-embedded submanifold with reach $\tau_c := \operatorname{reach}(\mathcal M_c) > 0$ 
and intrinsic dimension $m_c$. Let $\mu_c$ be a probability measure on 
$\mathcal M_c$ with density (with respect to the volume form) bounded below by 
$\rho_{\min} > 0$.

For any finite point set $W \subset \mathcal M_c$, define its covering radius
\[
    \delta(W) := \sup_{x \in \mathcal M_c}\, \min_{w \in W} \|x - w\|_2,
\]
and let $\mathcal T(W) = (W, E_{\mathrm{MST}}(W))$ be the Euclidean minimum 
spanning tree on $W$, with geometric realization
\[
    |\mathcal T(W)| := \bigcup_{(u,v) \in E_{\mathrm{MST}}(W)} [u, v] \;\subset\; \mathbb R^d.
\]
The Hausdorff distance between two compact sets $A, B \subset \mathbb R^d$ is
\[
    d_H(A, B) := \max\!\left\{
        \sup_{a \in A} \inf_{b \in B}\|a-b\|_2,\ \ 
        \sup_{b \in B} \inf_{a \in A}\|a-b\|_2
    \right\}.
\]

Let $U_c = \{u_1, \ldots, u_{N_c}\} \subset \mathcal M_c$ denote the full real 
sample set for class $c$, drawn i.i.d.\ from $\mu_c$, and let $V_c \subseteq U_c$ 
with $|V_c| = n$ be a uniform random subset. By the exchangeability of i.i.d.\ 
samples, $V_c$ is itself distributed as an i.i.d.\ sample of size $n$ from 
$\mu_c$, a characterization we adopt in the analysis below.

The aim of this subsection is to show that, with high probability, the subsample 
MST $|\mathcal T(V_c)|$ approximates the full MST $|\mathcal T(U_c)|$ in 
Hausdorff distance at a rate determined only by the intrinsic dimension $m_c$, 
justifying the practical use of small random subsamples in the geometric reward 
computation (Algorithm~1 of the main paper).

\begin{theorem}[Subsample-MST approximation]
\label{thm:subsample-mst}
Let $\mathcal M_c$, $\mu_c$, $U_c$, and $V_c$ be as above. Then there exist 
constants $C_{\mathcal M_c} > 0$ (depending only on $\mathcal M_c$ and 
$\rho_{\min}$) and $n_0 = n_0(\mathcal M_c, \tau_c)$ such that, for every 
$n \geq n_0$ and every $\eta \in (0, 1)$, with probability at least $1 - \eta$,
\begin{equation}
    \delta(V_c) \;\leq\; C_{\mathcal M_c}\!\left(\frac{\log(n/\eta)}{n}\right)^{\!1/m_c} \;<\; \tau_c,
    \label{eq:covering-bound}
\end{equation}
and consequently
\begin{equation}
    d_H\!\bigl(|\mathcal T(V_c)|,\ |\mathcal T(U_c)|\bigr) 
    \;\leq\; 2\,C_{\mathcal M_c}\!\left(\frac{\log(n/\eta)}{n}\right)^{\!1/m_c}.
    \label{eq:hausdorff-bound}
\end{equation}
Moreover, the geometric reward $r_{\mathrm{geo}}$ defined in Eq.~(4) of the 
main paper satisfies
\begin{equation}
    \bigl|r_{\mathrm{geo}}^{V_c}(\tilde x_0, c) - r_{\mathrm{geo}}^{U_c}(\tilde x_0, c)\bigr|
    \;\leq\; 2\,C_{\mathcal M_c}\!\left(\frac{\log(n/\eta)}{n}\right)^{\!1/m_c} + \lambda_{\mathrm{dir}},
    \label{eq:reward-stability}
\end{equation}
for every generated latent $\tilde x_0$, where $r_{\mathrm{geo}}^W$ denotes the 
geometric reward computed using the MST on point set $W$. The first term 
captures the contribution of the distance component of $r_{\mathrm{geo}}$ and 
vanishes as $n \to \infty$; the second term is a worst-case bound on the 
direction component (see~\cref{rem:direction-limitation}).
\end{theorem}

The proof proceeds in three steps: 
(\textit{i}) a probabilistic bound on the covering radius of the random 
subsample (\cref{lem:covering-rate}); 
(\textit{ii}) a deterministic Hausdorff bound between the two MSTs via the 
triangle inequality, leveraging \cref{thm:mst}; 
(\textit{iii}) a Lipschitz argument transferring the geometric bound to the 
distance component of the geometric reward, combined with a worst-case bound 
on the direction component.

\paragraph{Step 1: Covering radius of a random subsample.}
We first invoke a classical result on the convergence of empirical covering 
radius on a manifold.

\begin{lemma}[Covering radius rate, \cite{niyogi2008finding,penrose2003random}]
\label{lem:covering-rate}
Under the assumptions of~\cref{thm:subsample-mst}, there exists a constant 
$C_{\mathcal M_c} > 0$ depending only on $\mathcal M_c$ and $\rho_{\min}$ such 
that, for every $\eta \in (0, 1)$, with probability at least $1 - \eta$,
\begin{equation}
    \delta(V_c) \;\leq\; C_{\mathcal M_c}\!\left(\frac{\log(n/\eta)}{n}\right)^{\!1/m_c}.
    \label{eq:nsw-bound}
\end{equation}
\end{lemma}

\begin{proof}
Cover $\mathcal M_c$ by $\mathcal N(\rho)$ geodesic balls of radius $\rho$, 
where 
\[
\mathcal N(\rho) \leq c_1\, \rho^{-m_c}\, \mathrm{vol}(\mathcal M_c)
\]
for some constant $c_1 = c_1(\mathcal M_c)$, by standard volume comparison on 
bounded-reach submanifolds. For $\rho < \tau_c / 2$, each such ball has 
$\mu_c$-mass at least $\rho_{\min}\, c_2\, \rho^{m_c}$ by the lower density 
assumption combined with the small-ball volume bound on $\mathcal M_c$.

Setting 
\[
\rho := \!\left(\frac{\log(n/\eta)}{c_3\, n}\right)^{\!1/m_c}
\]
for an appropriate constant $c_3$, and applying a union bound over the cover, 
the probability that some ball contains no point of $V_c$ is at most
\[
\mathcal N(\rho)\,\bigl(1 - \rho_{\min}\, c_2\, \rho^{m_c}\bigr)^n \;\leq\; \eta.
\]
On this event, every $x \in \mathcal M_c$ lies within $\rho$ of some $v_i \in V_c$, 
yielding~\cref{eq:nsw-bound} with $C_{\mathcal M_c} := c_3^{-1/m_c}$.
\end{proof}

Choosing $n_0$ large enough that the right-hand side of~\cref{eq:nsw-bound} is 
strictly less than $\tau_c$, we obtain~\cref{eq:covering-bound} for all 
$n \geq n_0$.

\paragraph{Step 2: Hausdorff distance between the two MST realizations.}
Apply~\cref{thm:mst} to the point set $W = U_c$ and to $W = V_c$. Since 
$V_c \subseteq U_c$, the covering radius is monotone non-increasing under 
enlargement of the point set, i.e., $\delta(U_c) \leq \delta(V_c)$. Hence, on 
the event of~\cref{lem:covering-rate},
\[
\delta(U_c) \;\leq\; \delta(V_c) \;<\; \tau_c,
\]
so~\cref{thm:mst} applies to both sets and gives
\[
d_H\!\bigl(\mathcal M_c,\ |\mathcal T(U_c)|\bigr) \leq \delta(U_c), \qquad 
d_H\!\bigl(\mathcal M_c,\ |\mathcal T(V_c)|\bigr) \leq \delta(V_c).
\]
The triangle inequality for Hausdorff distance yields
\begin{align*}
    d_H\!\bigl(|\mathcal T(V_c)|,\ |\mathcal T(U_c)|\bigr)
    &\;\leq\; d_H\!\bigl(|\mathcal T(V_c)|,\ \mathcal M_c\bigr) + d_H\!\bigl(\mathcal M_c,\ |\mathcal T(U_c)|\bigr) \\
    &\;\leq\; \delta(V_c) + \delta(U_c) \\
    &\;\leq\; 2\,\delta(V_c).
\end{align*}
Combined with~\cref{eq:covering-bound}, this 
establishes~\cref{eq:hausdorff-bound}.

\paragraph{Step 3: Stability of the geometric reward.}
Fix a generated latent $\tilde x_0$ with class label $c$, and let 
$u_{\mathrm{gen}} = g_c(z_0)$ denote its projected representation as in 
Sec.~4.1 of the main paper. The geometric reward (Eq.~(4) of the main paper) 
admits the decomposition
\[
    r_{\mathrm{geo}}^W(\tilde x_0, c) 
    = -\,D_W \;-\;\lambda_{\mathrm{dir}}\, A_W,
\]
where
\[
    D_W := \min_{e \in E_{\mathrm{MST}}(W)} \mathrm{dist}(u_{\mathrm{gen}}, e), 
    \qquad 
    A_W := 1 - |\cos(\mathbf v_{\mathrm{gen}}, \mathbf v^*_{e_W})|,
\]
and $e_W^* \in E_{\mathrm{MST}}(W)$ is the MST edge nearest to 
$u_{\mathrm{gen}}$ in $W$.

\emph{Distance term.} Since $|\mathcal T(V_c)|$ and $|\mathcal T(U_c)|$ are 
compact subsets of $\mathbb R^d$, and the map 
$w \mapsto \mathrm{dist}(u_{\mathrm{gen}}, w)$ is $1$-Lipschitz,
\begin{equation}
    \bigl|D_{V_c} - D_{U_c}\bigr| 
    \;=\; \Bigl|\inf_{w \in |\mathcal T(V_c)|}\!\|u_{\mathrm{gen}} - w\| 
            - \inf_{w \in |\mathcal T(U_c)|}\!\|u_{\mathrm{gen}} - w\|\Bigr| 
    \;\leq\; d_H\!\bigl(|\mathcal T(V_c)|,\ |\mathcal T(U_c)|\bigr).
    \label{eq:distance-stability}
\end{equation}

\emph{Direction term.} Since $|\cos(\cdot, \cdot)| \in [0, 1]$, we have the 
worst-case bound
\begin{equation}
    \bigl|A_{V_c} - A_{U_c}\bigr| \;\leq\; 1.
    \label{eq:direction-stability}
\end{equation}
This bound cannot in general be improved through Hausdorff-distance arguments 
alone: for class manifolds of intrinsic dimension $m_c \geq 2$, two MSTs whose 
geometric realizations are close as point sets may nonetheless contain edges 
whose directions span different subspaces of the local tangent plane, so the 
nearest-edge direction $\mathbf v^*_{e_W}$ is not a Hausdorff-continuous 
functional of $|\mathcal T(W)|$ (see~\cref{rem:direction-limitation}).

Combining the two,
\[
    \bigl|r_{\mathrm{geo}}^{V_c} - r_{\mathrm{geo}}^{U_c}\bigr|
    \;\leq\; |D_{V_c} - D_{U_c}| + \lambda_{\mathrm{dir}}\,|A_{V_c} - A_{U_c}|
    \;\leq\; d_H\!\bigl(|\mathcal T(V_c)|,\ |\mathcal T(U_c)|\bigr) + \lambda_{\mathrm{dir}}.
\]
Substituting~\cref{eq:hausdorff-bound} 
yields~\cref{eq:reward-stability}, which completes the proof 
of~\cref{thm:subsample-mst}. \qed

\paragraph{Implemented distance term.}
For the nearest-node distance used in practice, \cref{prop:node-deployed}(iii) gives, deterministically and in the deployed representation,
\[
0 \;\le\; D^{V_c}_{\mathrm{node}}(u_{\mathrm{gen}}) - D^{U_c}_{\mathrm{node}}(u_{\mathrm{gen}}) \;\le\; \delta(V_c),
\]
so the distance component of the implemented reward changes by at most $\lambda_{\mathrm{dist}}\,\delta(V_c)$ under subsampling. Where the assumptions of \cref{lem:covering-rate} hold, $\delta(V_c) \le C_{\mathcal M_c}(\log(n/\eta)/n)^{1/m_c}$ with probability at least $1-\eta$. This bound does not depend on the MST edges, so it is unaffected by the $k$NN restriction; the direction term remains uncontrolled, as discussed next.

\begin{remark}[On the direction component]
\label{rem:direction-limitation}
The bound~\cref{eq:reward-stability} retains a constant term $\lambda_{\mathrm{dir}}$ 
that does not vanish with $\delta(V_c) \to 0$. This reflects an intrinsic 
limitation: while the distance component $D_W$ depends only on 
$|\mathcal T(W)|$ as a point set and is therefore controlled by Hausdorff 
distance, the direction component $A_W$ depends on the orientation of the 
nearest MST edge, which is not preserved under Hausdorff convergence on 
manifolds of intrinsic dimension $\geq 2$. In our experiments we set 
$\lambda_{\mathrm{dir}} = 0.6$ (chosen via ablation; cf.\ Fig.~3(b) of the 
main paper), so the absolute size of the direction contribution remains small 
relative to the dynamic range of the reward, preserving training stability. 
A finer analysis of the direction component, e.g.\ in expectation over 
$u_{\mathrm{gen}}$, is left to future work.
\end{remark}

\bigskip

The stability of the geometric reward guarantees that the excess-risk bound 
of the main paper carries over with $\delta$ replaced by $\delta(V_c)$, up to 
the additive direction-component constant.

\begin{corollary}[Excess risk under subsampled MST]
\label{cor:excess-risk-subsample}
Under the assumptions of~\cref{thm:subsample-mst} together with those of the 
main-paper excess-risk bound, the classifier $\hat h_S$ trained on the 
MGPO-distilled set generated using the subsample-MST satisfies, with 
probability at least $1 - \eta - \delta_0$,
\begin{align*}
    R_\mu(\hat h_S) - \min_{h^* \in \mathcal H} R_\mu(h^*) 
    \;&\leq\; 2\, L_\ell\!\left(\varepsilon + \frac{\delta(V_c)^2}{\tau_c} + \lambda_{\mathrm{dir}}\right) 
    + 4\,\mathfrak R_{N_S}(\ell\circ\mathcal H) \\
    &+ 2\, L_\ell\, C_{\mathcal M}\, N_T^{-1/m} 
    + \mathcal{O}\!\left(B\sqrt{\tfrac{\log(1/\delta_0)}{N_S}}\right).
\end{align*}
where $\delta(V_c) \leq C_{\mathcal M_c}\bigl(\log(n/\eta)/n\bigr)^{1/m_c}$.
\end{corollary}

\begin{proof}
Substitute the reward stability bound from~\cref{eq:reward-stability} into the 
excess-risk decomposition: the $\delta$-dependent term is replaced by 
$\delta(V_c)$, and the direction-component constant $\lambda_{\mathrm{dir}}$ 
contributes additively through the Lipschitz constant $L_\ell$. The result 
follows from a union bound over the high-probability events 
of~\cref{thm:subsample-mst} and the main-paper excess-risk bound.
\end{proof}

\begin{corollary}[Sample complexity for fixed precision, independent of $N_c$]
\label{cor:subsample-size}
Fix any target covering radius $\delta_{\mathrm{target}} \in (0, \tau_c)$ and 
any $\eta \in (0, 1)$. Then with probability at least $1 - \eta$, 
\[
    n \;=\; \Theta\!\left(
        \frac{\log(1/\eta) + m_c \log(1/\delta_{\mathrm{target}})}
             {\delta_{\mathrm{target}}^{m_c}}
    \right)
\]
random samples suffice to ensure $\delta(V_c) \leq \delta_{\mathrm{target}}$. 
In particular, the required $n$ depends only on the manifold parameters 
$(\tau_c, m_c, \rho_{\min})$ and the target precision $\delta_{\mathrm{target}}$, 
and is \textbf{independent of the full dataset size $N_c$}.
\end{corollary}

\begin{proof}
Set $\delta(V_c) = \delta_{\mathrm{target}}$ in~\cref{eq:nsw-bound} and solve 
for $n$:
\[
\delta_{\mathrm{target}} = C_{\mathcal M_c}\left(\tfrac{\log(n/\eta)}{n}\right)^{1/m_c} 
\;\Longrightarrow\; \tfrac{n}{\log(n/\eta)} = \Theta\bigl(\delta_{\mathrm{target}}^{-m_c}\bigr).
\]
Inverting gives the stated scaling. None of the constants involved depend on $N_c$.
\end{proof}

\begin{remark}[Practical implications for ImageNet-scale data]
\label{rem:imagenet-scaling}
The empirical intrinsic dimension of natural image manifolds, estimated by 
maximum-likelihood methods, lies in the range $m_c \in [26, 43]$ for ImageNet-scale 
data~\cite{pope2021intrinsic}, with class-conditional manifolds typically 
exhibiting smaller intrinsic dimension due to reduced intra-class variation. 
By~\cref{cor:subsample-size}, the subsample size $n$ required to reach a fixed 
geometric precision $\delta_{\mathrm{target}}$ is independent of $N_c$, so a 
single choice of $n$ generalizes across datasets of different scale. In our 
implementation, taking $n = 200$ random samples per class gives a ratio of the two covering-radius rates of $(N_c/200)^{1/m_c}$ (ignoring logarithmic factors), which evaluates 
to $\approx 1.07$–$1.13$ on ImageNet-1K ($N_c \approx 1300$, $m_c \in [26, 43]$). 
This is a heuristic comparison of upper-bound rates rather than a bound on the realized covering radii, and the intrinsic-dimension estimates refer to image space rather than the VAE/PCA space. It is consistent with, but does not by itself establish, the empirical scalability of MGPO to ImageNet-1K (Table~\ref{tab:imagenet1k}).
\end{remark}
\section{Related Work}
\subsection{Optimization-Based Dataset Distillation}
\label{sec: related}
Dataset distillation was originally introduced as a meta-learning problem that directly optimizes a small set of synthetic examples such that models trained on them achieve comparable performance to training on the full dataset \cite{wang2018dataset}.
A dominant line of follow-up work formulates distillation as bi-level optimization, where the outer loop updates synthetic data and the inner loop trains a network on these synthetic samples\cite{GM}.
Subsequent studies reduce accumulative error in optimization process by matching training trajectory.
Later, SRe$^2$L~\cite{Sre2L} series of work further unleashed the ability of decoupling the distillation process by squeeze and recover approach.
Besides, an alternative line of works matches higher-level statistics such as distribution or feature in representation space to avoid expensive unrolling, improving efficiency while sometimes sacrificing fidelity of training dynamics.
Nevertheless, these optimization-based methods often suffer from heavy compute or memory overhead due to unrolling, sensitivity to architecture choices, and difficulties in scaling to higher resolutions or large numbers of classes.

\subsection{Diffusion Models for Dataset Distillation}
Recent advances have explored diffusion models as powerful generative backbones for dataset distillation, enabling the synthesis of compact yet highly representative datasets. As a pioneering work, Minimax Diffusion~\cite{minimax} leverages pre-trained diffusion models and further fine-tunes them under a minimax criteria to jointly enhance sample representativeness and diversity. Building upon this idea, D$^4$M~\cite{D4M} incorporates pre-trained text-to-image diffusion models with prototype learning, clustering latent features to learn category prototypes and thereby improving intra-class consistency and visual fidelity. To mitigate distribution mismatch and stochasticity in latent sampling, D3HR~\cite{D3HR} proposes a domain mapping framework that applies DDIM inversion~\cite{DDIM} to project VAE latents into a Gaussian space with improved normality, combined with a group-based sampling strategy to better align the generated data distribution with the real one. More recently, CaO2~\cite{cao2} rectifies feature and class level inconsistencies by inference time latent optimization. DDVLCP~\cite{ddvlcp} extends diffusion-based distillation to the vision–language setting by constructing paired image–text prototypes using large language models, which effectively enhances semantic coherence and reduces artifacts such as missing or illogical objects.
Despite their success, diffusion-based approaches are inherently likelihood-centric, concentrating samples in high-density regions, which preserves perceptual fidelity but neglects boundary and hard examples essential for discriminative training.

\subsection{Reinforcement Learning in Diffusion models}

Reinforcement learning has recently been adopted to align diffusion and flow based generative models with task level objectives that are difficult to encode through likelihood based training alone. A representative line of work builds upon Group Relative Policy Optimization (GRPO)~\cite{shao2024deepseekmathpushinglimitsmathematical}, which updates generative policies by comparing relative rewards within groups of sampled trajectories, enabling stable optimization without explicit value function estimation.

Following the introduction of GRPO, several variants have been proposed to adapt it to different generative settings. Flow-GRPO~\cite{liu2025flowgrpo} extends group relative optimization to diffusion and flow matching models by introducing controlled stochasticity into the sampling process. Temporal variants such as Temp-Flow-GRPO~\cite{tempflow} further incorporate time dependent credit assignment to emphasize critical denoising stages. Mix-GRPO~\cite{li2026mixgrpounlockingflowbasedgrpo} combines multiple sampling dynamics within a unified optimization framework to balance exploration and efficiency. More recently, structured rollout strategies such as Branch-GRPO~\cite{li2025branchgrpostableefficientgrpo} and Dance-GRPO~\cite{xue2025dancegrpo} have been proposed to improve scalability and stability through branching trajectories and group wise reward aggregation.

Despite these advances, existing GRPO-based methods primarily focus on improving generative quality or preference alignment. 
In contrast, our work leverages reinforcement learning to explicitly incorporate downstream discriminative objectives into the diffusion generation process, addressing the objective mismatch between likelihood-driven diffusion models and dataset distillation for classification.

\section{Implementation Details}
\label{sec: supply}

\begin{algorithm}[t]
\caption{Per-class MST Construction in Latent Space}
\label{alg:mst}
\begin{algorithmic}[1]

\REQUIRE Real dataset $\mathcal{T}$, VAE encoder $E(\cdot)$,
max samples per class $N_c$, PCA dimension $d$, kNN size $k$

\ENSURE Class-wise MST skeleton $\{\mathcal{T}_c\}_c$ and PCA statistics $(\mu_c, W_c)$

\FOR{each class $c$}

\STATE Sample up to $N_c$ real images $\{x_i\}$ from class $c$

\STATE Encode latents $z_i = E(x_i)$ and flatten to vectors

\STATE Stack all latent vectors into matrix $Z$

\STATE Compute centered features $Z_c = Z - \mu_c$

\STATE Perform PCA projection to obtain normalized embeddings
$u_i = W_c(z_i - \mu_c)$

\STATE Construct a $k$-nearest-neighbor graph over $\{u_i\}$

\STATE Compute the minimum spanning tree $\mathcal{T}_c$
from the sparse adjacency graph

\ENDFOR

\end{algorithmic}
\end{algorithm}

\subsection{MST Construction Details}
\label{sec:supp_mst_details}

The class-wise MST skeleton is pre-computed once before ~\mgpo~ fine-tuning. 
For each class $c$, we sample up to $N_c$ real images from the training set. 
All images are resized to $256 \times 256$, center-cropped, converted to tensors, and normalized to the range $[-1,1]$. 
Each image is then encoded by a frozen Stable Diffusion VAE encoder. 
We use the mean of the encoder posterior as the latent representation and divide it by the standard latent scaling factor $0.18215$. 
The resulting latent tensor is flattened into a vector. 
Stacking all vectors from class $c$ produces a latent matrix
\[
Z \in \mathbb{R}^{N \times D},
\]
where $N$ is the number of collected samples and $D$ is the flattened latent dimension.

\noindent\textbf{Class-wise PCA normalization.}
To obtain a compact and numerically stable embedding space, we perform PCA on the centered latent matrix. 
Let
\[
\mu_c = \frac{1}{N}\sum_{i=1}^{N} z_i
\]
be the class mean and define the centered matrix
\[
Z_c = Z - \mu_c.
\]
We compute a low-rank PCA decomposition of $Z_c$ and retain the top $d$ principal components. 
Let $V \in \mathbb{R}^{D \times d}$ denote the projection matrix and $S \in \mathbb{R}^{d}$ the singular values. 
To avoid numerical instability, singular values are lower-bounded by a small constant $\varepsilon$. 
Each latent vector $z$ is then projected into a normalized embedding space by
\[
u = g_c(z) = (z-\mu_c)V\operatorname{diag}(S)^{-1}.
\]
This transformation simultaneously performs dimensionality reduction and whitening, making distances more comparable across directions. 
In our implementation, the PCA dimension is set to $d=128$.

\noindent\textbf{Sparse $k$NN graph construction.}
After obtaining the normalized embeddings $\{u_i\}_{i=1}^{N}$, we construct a sparse $k$-nearest-neighbor graph in the embedding space. 
For each node $u_i$, we identify its $k$ nearest neighbors according to Euclidean distance. 
An adjacency matrix $A$ is then constructed where
\[
A_{ij} =
\begin{cases}
\|u_i-u_j\|_2, & \text{if } u_j \text{ is among the } k \text{ nearest neighbors of } u_i, \\
+\infty, & \text{otherwise}.
\end{cases}
\]
The graph is symmetrized by assigning the same edge weight to both $A_{ij}$ and $A_{ji}$ once either direction is observed. 
This sparse construction significantly reduces graph density and improves scalability. 
In practice, we use $k=10$.

\noindent\textbf{Minimum spanning tree extraction.}
Given the sparse adjacency matrix, we compute the minimum spanning tree using a standard graph solver. 
The resulting tree is represented as an edge set
\[
\mathcal{E}_c = \{(i,j)\},
\]
which connects all nodes while minimizing the total edge weight. 
Ideally, a spanning tree over $N$ nodes contains $N-1$ edges. 
The extracted tree therefore provides a sparse skeleton of the class manifold that preserves global connectivity while avoiding redundant edges.

\noindent\textbf{Saved statistics.}
For each class, we store the PCA mean $\mu_c$, the projection matrix $V$, the singular values $S$, the normalized node embeddings, and the MST edge list. 
These quantities define both the class-wise tree $\mathcal{T}_c$ and the projection function $g_c(\cdot)$ used during ~\mgpo~ training. 
When a generated latent $z_0$ is obtained at the end of the diffusion trajectory, it is projected using the same transformation $g_c(\cdot)$ and evaluated according to its geometric relation to the nearest edge of the corresponding MST skeleton.

\subsection{Latent Geometric Reward Details}
\label{sec:supp_latent_reward}

We provide additional implementation details for the latent geometric reward $r_{\mathrm{geo}}$ used in ~\mgpo~. 
The reward is evaluated on a two-stage diffusion trajectory that combines a stochastic SDE transition with a deterministic ODE refinement. 
This design enables both stable policy optimization and geometry-aware trajectory supervision.

\noindent\textbf{Latent projection.}
To measure geometric consistency with the class manifold, both $z_{t_m}$ and $z_0$ are projected into the class-specific PCA space used for MST construction.

Let $g_c(\cdot)$ denote the projection function defined in Section~\ref{sec:supp_mst_details}. 
For a flattened latent vector $\mathrm{vec}(z)$, the projection is
\[
u = g_c(\mathrm{vec}(z)).
\]

Thus we obtain
\[
u_{t_m} = g_c(\mathrm{vec}(z_{t_m})), \qquad
u_0 = g_c(\mathrm{vec}(z_0)).
\]

Let $\mathcal{T}_c=(\mathcal{V}_c,\mathcal{E}_c)$ denote the class-wise MST skeleton constructed from real data.

\noindent\textbf{Distance-to-manifold term.}
The first component of the latent reward encourages the terminal latent to stay close to the class manifold represented by the MST nodes.

Let $\{v_i\}$ denote the nodes in $\mathcal{V}_c$. 
The distance penalty is defined as

\[
\mathcal{L}_{\mathrm{dist}}
=
\min_i \|u_0 - v_i\|_2 .
\]

This term measures how far the generated latent deviates from the learned manifold skeleton.

\noindent\textbf{Direction alignment term.}
To additionally constrain the diffusion trajectory, we measure whether the local latent transition direction aligns with the direction of nearby MST edges.

The sample transition direction in the projected space is

\[
d_{\mathrm{sample}} =
\frac{u_0 - u_{t_m}}{\|u_0 - u_{t_m}\|_2}.
\]

For each MST edge $(i,j)\in\mathcal{E}_c$, we define

\[
m_{ij} = \frac{v_i + v_j}{2}, 
\qquad
d_{ij} = \frac{v_j - v_i}{\|v_j - v_i\|_2}.
\]

We first find the nearest edge midpoint to the intermediate latent

\[
(i^\ast,j^\ast)
=
\arg\min_{(i,j)\in\mathcal{E}_c}
\|u_{t_m} - m_{ij}\|_2.
\]

The direction penalty is then defined as

\[
\mathcal{L}_{\mathrm{dir}}
=
1 -
\langle d_{\mathrm{sample}}, d_{i^\ast j^\ast} \rangle .
\]

This term becomes small when the local diffusion update follows the tangent direction of the class manifold.

\noindent\textbf{Final latent reward.}
The geometric reward is defined as the negative weighted sum of the two penalties:

\[
r_{\mathrm{geo}}
=
-\lambda_{\mathrm{dist}}\mathcal{L}_{\mathrm{dist}}
-\lambda_{\mathrm{dir}}\mathcal{L}_{\mathrm{dir}} .
\]

Therefore, higher reward values correspond to smaller manifold deviation and better alignment with the manifold geometry.

\subsection{Pixel Classifier Reward Details}
\label{sec:supp_pixel_reward}

We provide additional implementation details for the pixel-space discriminative reward $r_{\mathrm{disc}}$ used in \mgpo.
The pixel reward evaluates the semantic consistency of generated samples in pixel space and is computed from the decoded image $\tilde{x}_0 = \mathcal{D}(z_0)$ corresponding to the terminal latent produced by the diffusion trajectory.

\noindent\textbf{Classification.}
For standard image classification, the discriminative reward is derived from a pretrained proxy classifier $f_\phi$.
Given a generated image $\tilde{x}_0$ and target class $c$, the reward is defined as the log-probability of the ground-truth class:
\begin{equation}
    r_{\mathrm{disc}}(\tilde{x}_0, c) = \log \left( \frac{\exp(f_\phi(\tilde{x}_0)_c)}{\sum_k \exp(f_\phi(\tilde{x}_0)_k)} \right),
\end{equation}
which provides a smooth reward landscape and penalizes samples near ambiguous decision boundaries.

\noindent\textbf{Object Detection.}
For object detection, the discriminative reward is derived from a pretrained Faster R-CNN~\cite{ren2016fasterrcnnrealtimeobject} with a ResNet-50 FPN backbone.
Given a generated image $\tilde{x}_0$ conditioned on class $c$, the reward is the maximum post-NMS confidence score for class $c$, as defined in ~\cref{eq: r_det} of the main paper.
The full computation is described in Algorithm~\ref{alg:det_reward}.

\begin{algorithm}[t]
\caption{Detection Reward Computation}
\label{alg:det_reward}
\begin{algorithmic}[1]
\REQUIRE Generated image $\tilde{x}_0$, target class $c$, pretrained detector $f_\phi$ (Faster R-CNN)
\ENSURE Scalar reward $r_\text{det}$
\STATE $\text{proposals} \leftarrow f_\phi.\text{region\_proposals}(\tilde{x}_0)$
\STATE $\text{detections} \leftarrow \text{NMS}(\text{proposals},\ \tau_\text{iou}{=}0.5,\ \tau_\text{score}{=}0.05)$
\STATE $\mathcal{S}_c \leftarrow \{\text{score}(d) \mid d \in \text{detections},\ \text{class}(d) = c\}$
\IF{$\mathcal{S}_c = \emptyset$}
    \STATE $r_\text{det} \leftarrow -1.0$ \hfill $\triangleright$ no valid detection for class $c$
\ELSE
    \STATE $r_\text{det} \leftarrow \max(\mathcal{S}_c)$ \hfill $\triangleright$ ~\cref{eq: r_det} of main paper
\ENDIF
\RETURN $r_\text{det}$
\end{algorithmic}
\end{algorithm}

\noindent\textbf{Semantic Segmentation.}
For semantic segmentation, the discriminative reward is derived from a pretrained LR-ASPP~\cite{howard2019searchingmobilenetv3} with a MobileNetV3-Large backbone.
Given a generated image $\tilde{x}_0$ conditioned on class $c$, the reward is the mean confidence over the top-$K$ activated pixels of the spatial probability map for class $c$, as defined in ~\cref{eq: r_seg} of the main paper.
The TopK aggregation prevents reward hacking by enforcing a minimum expected object area, discouraging the policy from exploiting isolated high-confidence pixels.
The full computation is described in Algorithm~\ref{alg:seg_reward}.

\begin{algorithm}[t]
\caption{Segmentation Reward Computation}
\label{alg:seg_reward}
\begin{algorithmic}[1]
\REQUIRE Generated image $\tilde{x}_0$, target class $c$, pretrained segmentation network $f_\phi$ (LR-ASPP), TopK fraction $\alpha$
\ENSURE Scalar reward $r_\text{seg}$
\STATE $\mathcal{P}_c \leftarrow f_\phi(\tilde{x}_0)[c] \in \mathbb{R}^{H \times W}$ \hfill $\triangleright$ softmax probability map for class $c$
\STATE $\mathbf{p} \leftarrow \mathrm{flatten}(\mathcal{P}_c) \in \mathbb{R}^{H \cdot W}$
\STATE $K \leftarrow \lfloor \alpha \cdot H \cdot W \rfloor$ \hfill $\triangleright$ number of top pixels
\IF{$\max(\mathbf{p}) < \epsilon$}
    \STATE $r_\text{seg} \leftarrow -1.0$ \hfill $\triangleright$ no activation for class $c$
\ELSE
    \STATE $r_\text{seg} \leftarrow \frac{1}{K} \sum_{k=1}^{K} \mathrm{TopK}(\mathbf{p})_k$ \hfill $\triangleright$ ~\cref{eq: r_seg} of main paper
\ENDIF
\RETURN $r_\text{seg}$
\end{algorithmic}
\end{algorithm}
\section{Additional Experiments}
\label{sec: addexp}

\subsection{Experiment Results on Different Student Model}
\begin{table*}[t]
\centering
\caption{Performance comparison with soft-label-based baselines on ImageWoof under different IPC and backbone settings. The best results in each setting are highlighted in bold.}
\resizebox{\textwidth}{!}{
\begin{tabular}{c|c|cccc|c}
\toprule
Backbone & IPC
& SRe$^2$L~\cite{Sre2L}
& Minimax~\cite{minimax}
& RDED~\cite{RDED}
& CaO$^2$~\cite{cao2}
& \mgpo~(Ours) \\
\midrule

\multirow{2}{*}{ResNet-18}
& 10 & 20.2{\tiny$\pm$0.2} & 40.1{\tiny$\pm$1.0} & 38.5{\tiny$\pm$2.1} & 45.6{\tiny$\pm$1.4} & \textbf{47.3{\tiny$\pm$0.8}} \\
& 50 & 23.3{\tiny$\pm$0.3} & 67.0{\tiny$\pm$1.8} & 68.5{\tiny$\pm$0.7} & 68.9{\tiny$\pm$1.1} & \textbf{69.7{\tiny$\pm$1.6}} \\
\midrule

\multirow{2}{*}{ResNet-50}
& 10 & 17.3{\tiny$\pm$1.7} & 37.3{\tiny$\pm$1.1} & 29.9{\tiny$\pm$2.2} & 40.1{\tiny$\pm$0.1} & \textbf{41.3{\tiny$\pm$1.1}} \\
& 50 & 24.8{\tiny$\pm$0.7} & 64.3{\tiny$\pm$0.9} & 67.8{\tiny$\pm$0.3} & 68.2{\tiny$\pm$1.1} & \textbf{69.0{\tiny$\pm$0.4}} \\
\midrule

\multirow{2}{*}{ResNet-101}
& 10 & 17.7{\tiny$\pm$0.9} & 34.2{\tiny$\pm$1.7} & 31.3{\tiny$\pm$1.3} & 36.5{\tiny$\pm$1.4} & \textbf{37.8{\tiny$\pm$1.6}} \\
& 50 & 21.2{\tiny$\pm$0.2} & 62.7{\tiny$\pm$1.6} & 59.1{\tiny$\pm$0.7} & 63.1{\tiny$\pm$1.3} & \textbf{63.5{\tiny$\pm$1.5}} \\
\bottomrule
\end{tabular}\label{backbone}
}
\end{table*}

As shown in Table~\ref{backbone}, we report the performance of different distillation methods under the soft-label based training protocol.
Experiments are conducted on ImageWoof across multiple backbone architectures, including ResNet-18, ResNet-50, and ResNet-101, under various IPC settings.
Across all settings, \mgpo~consistently achieves the best performance.
Notably, the performance gains remain stable across different backbone capacities, indicating that the distilled datasets produced by \mgpo~generalize well to diverse student architectures. In particular, under low IPC setting (IPC = 10),~\mgpo~boosts ResNet-18 performance from 45.6\% (CaO$^2$) to 47.3\%, and ResNet-50 accuracy from 40.1\% to 41.3\%. Under higher IPC (IPC = 50),~\mgpo~ continues to outperform strong baselines such as Minimax~\cite{minimax}, RDED~\cite{RDED}, and CaO$^2$~\cite{cao2}. 
This suggests that our reward-guided distillation framework is able to capture class-discriminative information that transfers effectively across models of varying depth and capacity.

\subsection{Hyperparameter Analysis}
We further study the sensitivity of our method to several key hyperparameters. All experiments are conducted on ImageWoof with 10 IPC distilled datasets. Specifically, we analyze the parameters involved in the construction of the MST that models the latent geometric structure. By varying these parameters while keeping other settings fixed, we evaluate the robustness of our framework and examine how they influence the quality of the generated synthetic data and the downstream performance~(Figure~\ref{fig:mst_param}).

\begin{figure}[H]
    \centering
    \includegraphics[width=0.95\linewidth]{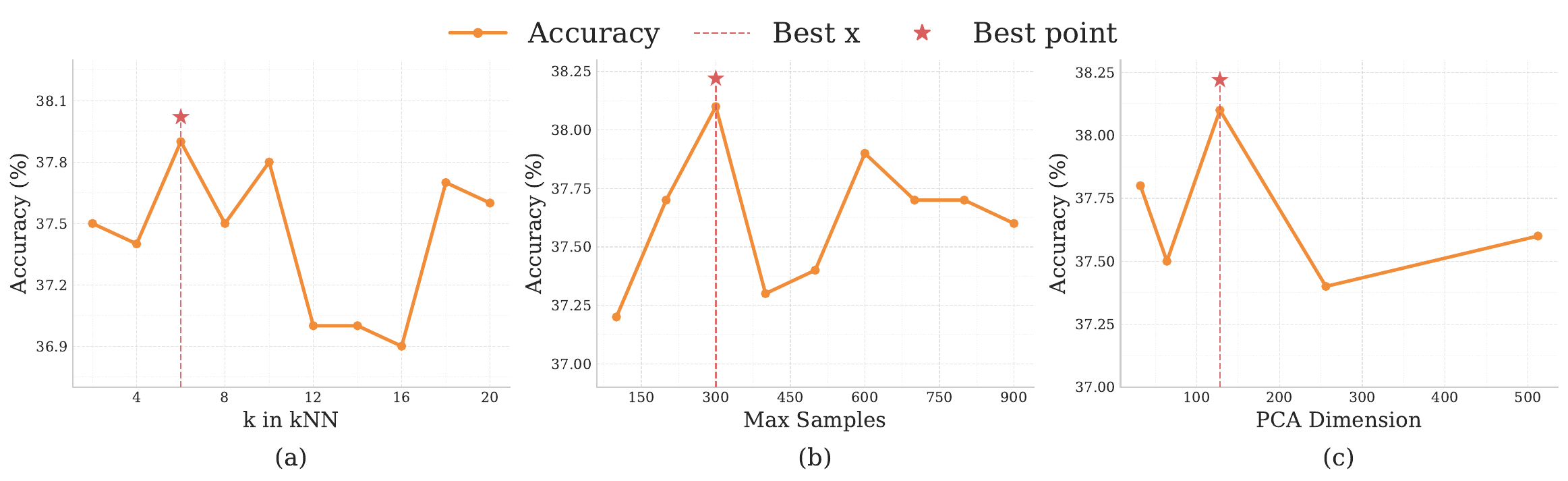}
    \caption{Ablation studies on the hyperparameters used in MST construction, including
the number of neighbors k in the kNN graph, the maximum number of sampled
latent features per class, and the PCA projection dimension.}
    \label{fig:mst_param}
\end{figure}

Overall, MGPO exhibits stable performance across a broad range of hyperparameter settings.
The results suggest that the method is not overly sensitive to individual parameters, and moderate values generally provide the best balance between optimization stability and reward effectiveness.

\section{Visualization}
\label{sec: vis}

\subsection{Results on Classification}
Figure~\ref{fig:woof_compare},~\ref{fig:nette_compare}, and \ref{fig:idc_compare} demonstrate the visualization results on ImageWoof, ImageNette and ImageIDC, respectively.
\begin{figure*}[t]
\centering
\includegraphics[
width=1.0\textwidth,
]{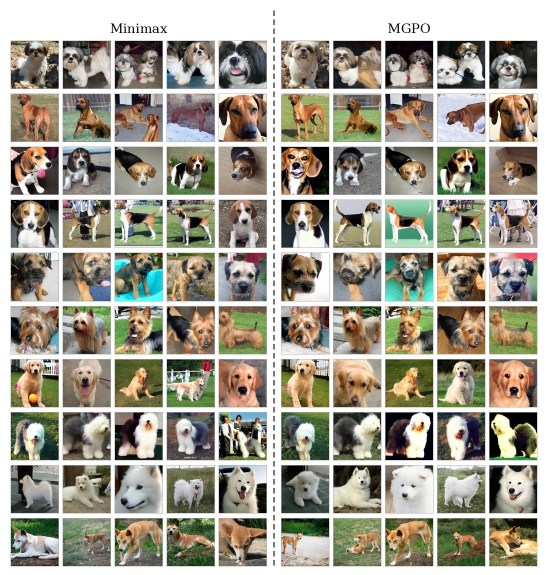}
\caption{Comparison between samples generated by Minimax Diffusion (left) and generated by the proposed \mgpo~method (right) on ImageWoof (Classification).}
\label{fig:woof_compare}
\end{figure*}

\begin{figure}
    \centering
    \includegraphics[width=1\linewidth]{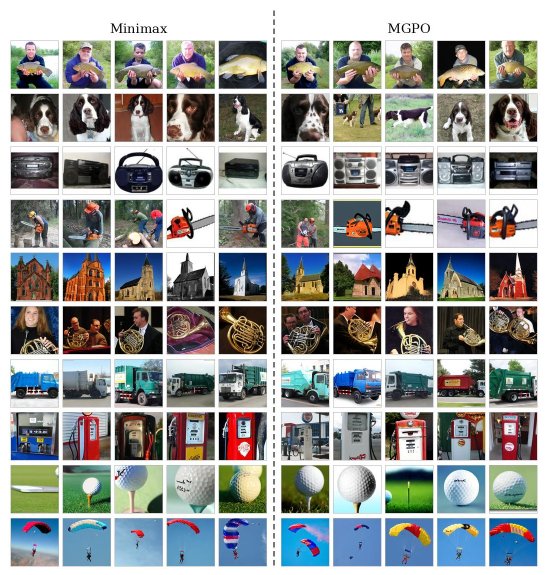}
    \caption{Comparison between samples generated by Minimax Diffusion (left) and generated by the proposed \mgpo~method (right) on ImageNette (Classification).}
    \label{fig:nette_compare}
\end{figure}

\begin{figure}
    \centering
    \includegraphics[width=1\linewidth]{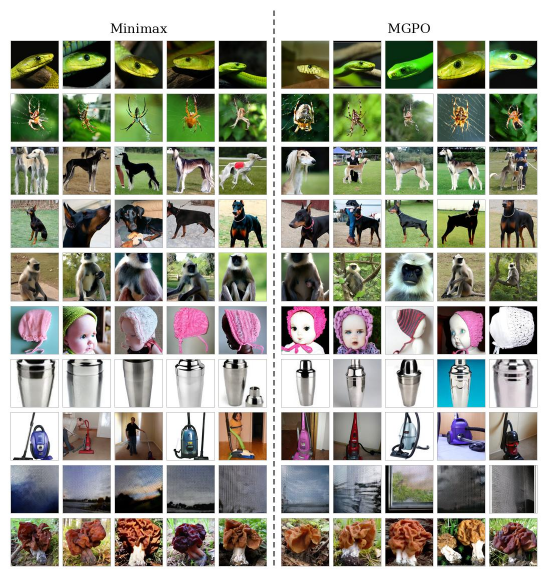}
    \caption{Comparison between samples generated by Minimax Diffusion (left) and generated by the proposed \mgpo~method (right) on ImageIDC (Classification).}
    \label{fig:idc_compare}
\end{figure}

\subsection{Results on Detection}
Figure~\ref{fig:detection} demonstrates the visualization results on detection results.
\begin{figure}[h]
    \centering
    \includegraphics[width=1.0\linewidth]{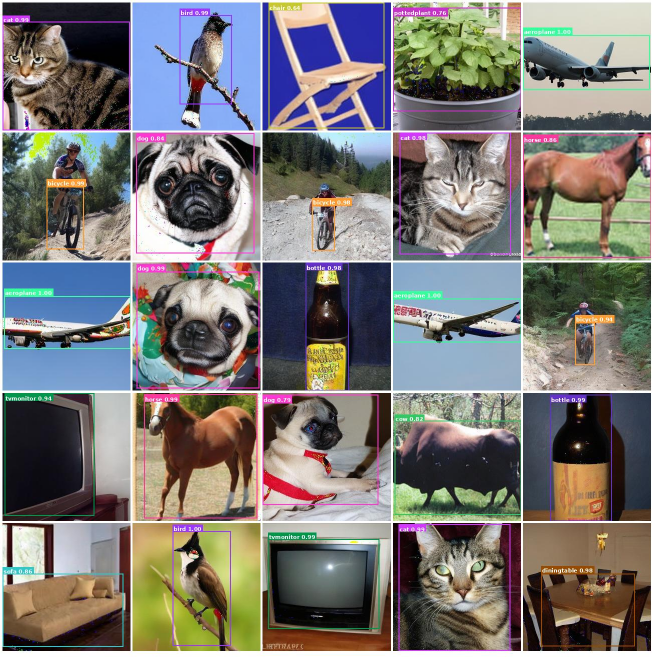}
    \caption{Samples generated by the proposed \mgpo~(Object Detection).}
    \label{fig:detection}
\end{figure}

\subsection{Results on Segmentation}
Figure~\ref{fig:segmentation} demonstrates the visualization results on segmentation results.
\begin{figure}[h]
    \centering
    \includegraphics[width=1.0\linewidth]{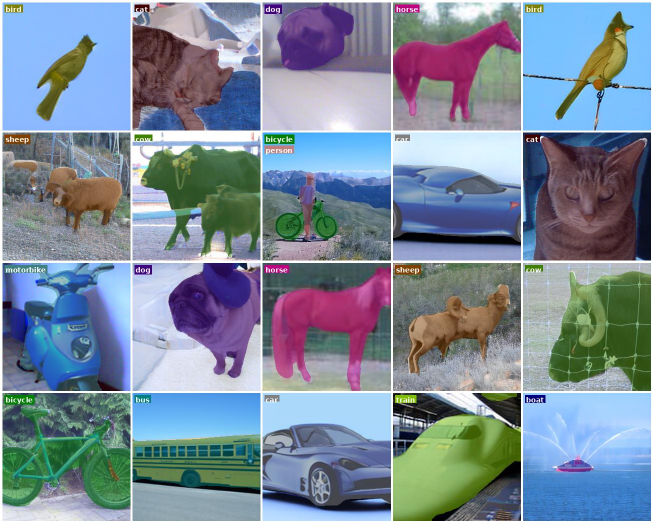}
    \caption{Samples generated by the proposed \mgpo~(Segmentation).}
    \label{fig:segmentation}
\end{figure}

\section{Limitations and Broader Impacts}
\label{sec:limitations_impacts}

\paragraph{Limitations.} 
While MGPO establishes an efficient and adaptable paradigm for task aware dataset distillation, it still possesses certain limitations. 
First, from a theoretical perspective, as discussed in \cref{rem:direction-limitation}, the direction component of the geometric reward relies on local tangent approximations which may exhibit fluctuations on high intrinsic dimension manifolds. 
Second, although MGPO eliminates the computational bottleneck of inference time guidance, the reinforcement learning policy optimization phase introduces additional offline training overhead compared to purely training free selection methods. 
Finally, scaling the class wise MST construction to datasets with millions of categories might require more advanced hierarchical clustering algorithms to maintain strict scalability.

\paragraph{Broader Societal Impacts.} 
Our research presents significant positive societal impacts by heavily compressing the data required to train deep learning models. This inherently reduces the carbon footprint and energy consumption associated with large scale model training, democratizing artificial intelligence by allowing researchers with limited computational resources to train state of the art models. 
However, as a potential negative impact, if the pre-trained reward models (such as the proxy classifier or detector) contain inherent biases regarding demographic attributes, the policy optimization process might internalize and amplify these biases within the distilled dataset. Therefore, practitioners must ensure that the reward models used for distillation are thoroughly audited for fairness and equity prior to deployment in sensitive applications.


\end{document}